\documentclass[10pt,journal,compsoc]{IEEEtran}
\usepackage[T1]{fontenc}
\usepackage{amsmath,amssymb,amsthm,mathtools}
\usepackage{amsmath,amsfonts,bm}

\def\eqref#1{equation~\ref{#1}}
\def\1{\bm{1}}

\DeclareMathAlphabet{\mathsfit}{\encodingdefault}{\sfdefault}{m}{sl}
\SetMathAlphabet{\mathsfit}{bold}{\encodingdefault}{\sfdefault}{bx}{n}

\usepackage{booktabs,multirow,array}
\usepackage{graphicx}
\usepackage{xcolor}
\usepackage{enumitem}
\usepackage{float}
\usepackage[numbers,sort]{natbib}
\setcitestyle{citesep={], [}}
\makeatletter
\renewcommand{\NAT@def@citea}{\def\@citea{\NAT@separator}}
\makeatother
\AddToHook{env/thebibliography/begin}{\interlinepenalty=10000}
\usepackage{url}
\usepackage[hidelinks,bookmarksnumbered]{hyperref}
\hypersetup{pdftitle={DISCERN: Identifiable Disease Evidence for Generalizable Medical Multimodal Learning},pdfauthor={Xiaoguang Zhu, Pengyi Jiang, Juncen Guo, Jingyi Wu, Lianlong Sun, Yang Liu}}
\newcommand{\method}{DISCERN}

\newcommand{\Bb}{\Gamma}
\newcommand{\Ee}{\mathcal{E}}
\newcommand{\qref}{q^{\mathrm{ref}}_{\pi}}

\renewcommand{\eqref}[1]{(\ref{#1})}
\theoremstyle{plain}
\newtheorem{theorem}{Theorem}
\newtheorem{proposition}{Proposition}

\theoremstyle{definition}

\theoremstyle{remark}

\newenvironment{fullstatement}[1]{\par\medskip\noindent\textbf{#1 (full statement).}\ \itshape}{\par\medskip}
\title{DISCERN: Identifiable Disease Evidence for Generalizable Medical Multimodal Learning}
\author{Xiaoguang~Zhu, Pengyi~Jiang, Juncen~Guo, Jingyi~Wu, Lianlong~Sun, and~Yang~Liu%
\IEEEcompsocitemizethanks{%
\IEEEcompsocthanksitem Xiaoguang Zhu is with Stanford University, Stanford, CA, USA, and the University of California, Davis, Davis, CA, USA. Part of this work was completed at the University of California, Davis.
\IEEEcompsocthanksitem Pengyi Jiang is with the National University of Singapore, Singapore.
\IEEEcompsocthanksitem Juncen Guo and Jingyi Wu are with Fudan University, Shanghai, China.
\IEEEcompsocthanksitem Lianlong Sun is with the University of Rochester, Rochester, NY, USA.
\IEEEcompsocthanksitem Yang Liu is with Tongji University, Shanghai, China.}}
\begin{document}
\bstctlcite{TPAMI:BSTcontrol}
\IEEEtitleabstractindextext{%
\begin{abstract}
Clinical multimodal models can learn enrollment shortcuts shared across modalities and fail on new cohorts. We model these data through a latent structural causal model: disease state and background factors generate physiology through shared mechanisms, while cohorts differ in selection and measurement. Holding acquisition fixed, we identify conditional disease evidence without observing disease-state labels or recovering the full causal model. In common Gaussian coordinates, cross-cohort state alignment and natural-parameter changes identify a minimal selection statistic; conditioning on it removes enrollment-dependent evidence and yields an identifiable reference posterior within the specified shift family. State-dependent interactions can carry evidence invisible to individual-modality posteriors, motivating joint, information-preserving representations. DISCERN combines an information-lattice encoder, a latent-state model, and an outcome head through joint training. Under reversed synthetic enrollment on known latent coordinates, conditioning raises state AUROC from 0.043 to 0.883 and from 0.265 to 0.825, while random conditioning does not. Exploratory cross-database mortality experiments between eICU and MIMIC-IV improve AUROC over matched empirical risk minimization by 0.036 and 0.035. Learned-coordinate diagnostics distinguish preservation of state information from identification of the coordinates required by the estimator.
\end{abstract}
}
\maketitle
\IEEEdisplaynontitleabstractindextext
\section{Introduction}
\label{sec:intro}

\IEEEPARstart{M}{edical} multimodal learning integrates imaging, laboratory values, vital signs, cognitive scores, demographics, and genetics to support clinical prediction \cite{acosta2022multimodal,kline2022multimodal,huang2020fusion,liu2025mednursingsurvey}. These observations provide complementary, partial views of a patient's underlying condition. Fusion architectures, cross-modal pre-training, and missing-modality handling have improved predictive performance \cite{hayat2022medfuse,soenksen2022integrated,liang2024foundations,elsharief2025medmod,gong2026aleatoric}, and theory explains when multimodal representations help \cite{huang2021what,lu2023theory,dufumier2025what}. Yet combining more information does not establish which of that information remains reliable when the model moves to a different cohort.

\begin{figure*}[t]
\centering
\includegraphics[width=0.8\linewidth]{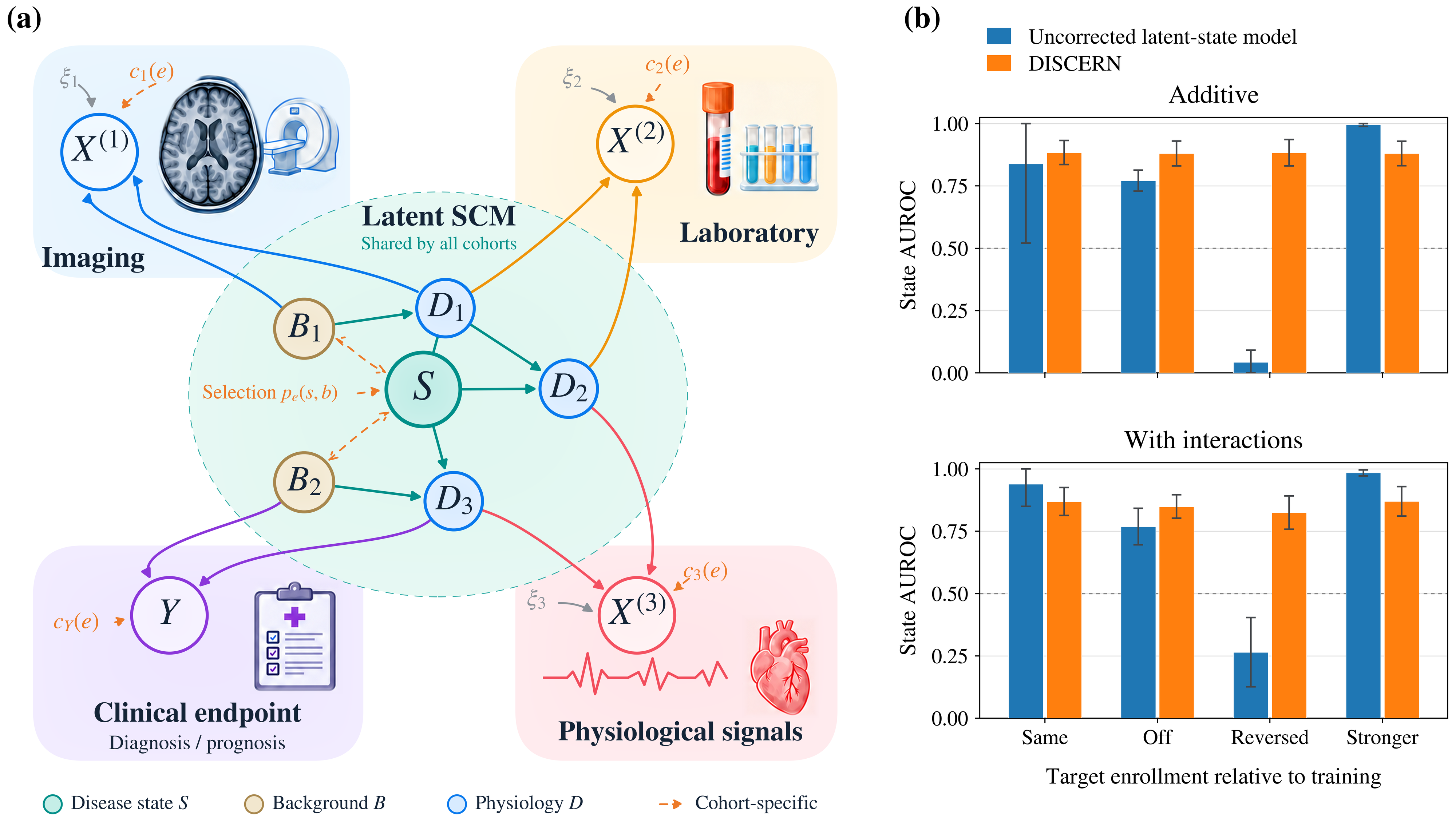}
\caption{\textbf{A generative view of medical multimodal data and enrollment-induced failure.} (a)~Disease state $S$ and background factors $B$ generate physiology $D$ through a latent structural causal model shared across cohorts. Each modality observes an overlapping subset of latent coordinates with private variation $\xi_m$; the endpoint $Y$ is another partial view. Orange indicates cohort-dependent selection $p_e(s,b)$ and measurement $c_m(e)$, $c_Y(e)$. (b)~On the generator's latent coordinates, target cohorts preserve, remove, reverse, or strengthen the source enrollment pattern. The uncorrected latent-state model follows the enrollment and collapses under reversal, whereas the conditional evidence of \method{} remains stable across all four settings (Table~\ref{tab:state}).}
\label{fig:motivation}
\end{figure*}

The difficulty is that clinical observations reflect both the patient and the process through which the patient enters a dataset. Cohorts differ in \emph{measurement}: scanners, assays, and labeling protocols change how a condition is recorded. They also differ in \emph{selection}: enrollment changes disease prevalence and the association between disease and background factors such as age. Because background factors also shape physiology, selection can change the distribution of each modality within a disease state even when the underlying physiological mechanisms remain fixed. Consider Alzheimer's disease assessment using structural MRI and cognitive scores. Both hippocampal volume and memory performance vary with age even among cognitively normal individuals~\cite{jack2015age}. If a source cohort enrolls older patients with Alzheimer's disease and younger controls, both modalities can reinforce an age-related shortcut that becomes unreliable in a target cohort with a different age--disease association. Cross-modal agreement therefore need not indicate stable disease evidence. This concern is consistent with documented shortcut learning and failures of external validation in imaging and clinical prediction~\cite{zech2018variable,degrave2021ai,wong2021external,finlayson2021clinician,allyn2026multisource,tranchellini2026evaluating}.

A further challenge is that the disease state itself is not directly observed. The recorded diagnosis or outcome is an imperfect, protocol-dependent view that can also depend on physiology beyond the state. This distinction complicates existing approaches to generalization. Domain-generalization methods seek stable predictors, risks, or feature statistics across training environments~\cite{arjovsky2019invariant,krueger2021out,sagawa2020distributionally,sun2016deep}, but their success depends on specifying an appropriate notion of invariance~\cite{rosenfeld2021risks,kamath2021does}; suppressing variation across cohorts can also remove disease information when prevalence and case mix change. Nuisance-based and causal methods make stability assumptions explicit~\cite{puli2022nuisance,wang2022provable,peters2016causal,londschien2026domain,zhu2025causaldebiasing,zhu2026robusthealthcare,zhu2026comedpo}, but typically formulate prediction around an observed target or anchor. Causal representation learning establishes conditions for identifying latent variables~\cite{khemakhem2020variational,kivva2022identifiability,yao2024multiview,sun2025causal}, yet latent-variable identification alone does not determine which evidence transfers under changing enrollment (Section~\ref{sec:related} discusses related work in detail). These limitations motivate our central question: \emph{Which evidence about an unobserved disease state is identifiable from differently assembled cohorts and remains valid when enrollment changes?}

We address this question through a generative model of medical multimodal data (Figure~\ref{fig:motivation}a). A latent disease state and background factors jointly generate physiology through mechanisms shared across cohorts. Modalities observe overlapping subsets of these latent quantities, while the recorded endpoint provides an additional partial view. Cohorts vary in selection and measurement; our identification analysis isolates selection by holding acquisition fixed. Figure~\ref{fig:motivation}b illustrates why this distinction matters: even on the generator's latent coordinates, an uncorrected state model can exploit the source enrollment pattern and fail when that pattern reverses. Recovering latent coordinates is therefore insufficient unless the evidence extracted from them is also separated from enrollment.

Our answer is \emph{conditional disease evidence}: the state log-likelihood ratio after subtracting the contribution carried by a \emph{selection statistic}. This statistic captures how enrollment changes the state-conditional distributions across cohorts. Conditioning on it removes this contribution while retaining the information needed to interpret physiology in context. In the Alzheimer's example, the aim is to interpret MRI findings and cognitive scores relative to a patient's age while excluding the enrollment-dependent age--disease association from the evidence. In common Gaussian coordinates with state-specific covariances shared across cohorts, we give conditions that align unobserved states up to one global permutation. Once aligned, their natural-parameter changes identify a minimal linear selection statistic and, in turn, the conditional evidence and a reference posterior. The same evidence transfers to target cohorts whose selection shifts lie within the identified family.

Estimating this evidence from multimodal observations also requires preserving the information on which it depends. Disease evidence may reside in interactions invisible to individual-modality posteriors or in latent factors observed by only a subset of modalities. Our analysis therefore motivates joint modeling and representations that retain shared, partially shared, and private information. These principles lead to \method{}, which combines an information-lattice encoder with an aligned latent-state model and a conditional likelihood-ratio correction. Because the recorded endpoint can depend on continuous physiology beyond the disease state, its outcome head uses both conditional evidence and continuous features. A separate risk decomposition characterizes why improved state evidence need not translate directly into improved endpoint prediction.

In summary, (i)~we formulate cohort variation through selection and measurement in a shared physiological model and, under explicit coordinate and alignment assumptions, identify a minimal linear selection statistic and the conditional disease evidence up to a global state permutation (Section~\ref{sec:theory}); (ii)~we develop \method{}, which estimates this evidence without disease-state labels and combines it with continuous physiology for outcome prediction (Section~\ref{sec:method}); and (iii)~on controlled enrollment shifts, conditioning on the estimated selection statistic restores state recovery under reversed enrollment where random conditioning does not, and on cross-database ICU mortality \method{} improves AUROC over matched ERM in both directions (Section~\ref{sec:experiments}).

\section{Related work}
\label{sec:related}

\subsection{Medical multimodal learning}
Fusion architectures for imaging and electronic health records span early, late and hybrid designs~\citep{huang2020fusion,acosta2022multimodal,kline2022multimodal}; MedFuse handles paired and unpaired time series and radiographs~\citep{hayat2022medfuse}, HAIM combines tabular, temporal, text and image embeddings~\citep{soenksen2022integrated}, and contrastive objectives pre-train imaging encoders from clinical tables~\citep{hager2023best}. Missing-modality transformers~\citep{ma2022multimodal,gong2026aleatoric,zhu2025causaldebiasing,zhu2026robusthealthcare} and latent-space imputation with M3Care~\citep{zhang2022m3care} address incomplete inputs, while modality-balancing objectives~\citep{wang2020what,peng2022balanced} address optimization imbalance. Recent work disentangles shared from modality-specific diagnostic information in paired ECG and cardiac MRI~\citep{tonekaboni2025representation}, treats modality missingness as informative through a latent health state, with experiments on MIMIC-IV and eICU~\citep{liang2025causal}, and benchmarks EHR and chest X-ray fusion within one health system~\citep{elsharief2025medmod}. These studies focus on fusion, representation quality, and modality availability. Their predictive gains do not by themselves establish invariant disease evidence under a change in cohort selection. Shortcut learning in medical imaging~\citep{zech2018variable,degrave2021ai,winkler2019association,oakdenrayner2020hidden,geirhos2020shortcut,brown2023detecting} and deployment failures~\citep{wong2021external,nestor2019feature} document the cost. \citet{huang2021what} explain when multiple modalities provably outperform a single one by the quality of a shared latent representation; \citet{lu2023theory} give a learning-theoretic account based on heterogeneity between modalities and a learnable connection between them. More recent analyses show that contrastive alignment captures mainly redundant information and can miss unique and synergistic terms~\citep{dufumier2025what}, that whether alignment helps depends on the balance of redundant and unique information~\citep{tjandrasuwita2025understanding}, and that modality competition can be rebalanced by a game-theoretic regularizer~\citep{kontras2025balancing}. These analyses address representation quality and multimodal optimization; the question here is which evidence about a latent disease state remains invariant under a specified family of cohort-selection shifts. These studies motivate modeling the acquisition process and spurious associations explicitly. The target in the present work is the identifiable conditional evidence about an unobserved disease state under a specified cross-cohort selection family.

\subsection{Robust prediction under distribution shift}
IRM~\citep{arjovsky2019invariant}, V-REx~\citep{krueger2021out}, GroupDRO~\citep{sagawa2020distributionally}, CORAL~\citep{sun2016deep}, LISA~\citep{yao2022improving}, JTT~\citep{liu2021just}, DFR~\citep{kirichenko2023last}, SWAD~\citep{cha2021swad} and test-time adaptation~\citep{iwasawa2021test,wang2021tent} form the standard baselines; DomainBed~\citep{gulrajani2021search} highlights the strength of ERM under a controlled evaluation protocol, while WILDS~\citep{koh2021wilds} benchmarks distribution shifts across diverse applications, and \citet{rosenfeld2021risks,kamath2021does} show when invariance criteria fail. NuRD~\citep{puli2022nuisance} and invariant-feature subspace recovery~\citep{wang2022provable} are the closest statistical targets; anchor regression~\citep{rothenhausler2021anchor}, invariant causal prediction~\citep{peters2016causal,rojascarulla2018invariant} and conditional variance penalties~\citep{heinzedeml2021conditional} condition on observed labels or anchors. Clinical DG frameworks~\citep{zhang2021empirical} and harmonization~\citep{johnson2007adjusting,fortin2018harmonization} address site effects; \citet{subbaswamy2019preventing,subbaswamy2020development,castro2020causality} argue for causal modelling of clinical shifts. In critical care, multi-database studies find that internal validation overestimates external performance and that site-specific features act as shortcuts for mortality prediction~\citep{allyn2026multisource}, that sepsis models degrade across MIMIC-IV, eICU and HiRID with the best adaptation strategy depending on target data~\citep{tranchellini2026evaluating}, and that anchor regression with the database as anchor aids transfer across nine ICU databases~\citep{londschien2026domain}.

Medical vision--language reliability raises a related distinction between useful context and shortcuts. CoMedPO~\citep{zhu2026comedpo} uses counterfactual preference optimization to mitigate background-related modality bias while preserving useful lesion--background dependencies in medical visual question answering and report generation. This provides a complementary example of retaining clinically informative relations while reducing spurious dependence. Its objective aligns a generative policy through preference supervision; \method{} identifies a conditional likelihood-ratio target for a latent disease state from variation in cohort selection.

\subsection{Identifiable multimodal and causal representations}
Causal representation learning builds on identifiability results for finite Gaussian mixtures~\citep{yakowitz1968identifiability}, nonlinear ICA with auxiliary variables~\citep{hyvarinen2016unsupervised,hyvarinen2019nonlinear,khemakhem2020variational}, mixture priors with piecewise-affine decoders~\citep{kivva2022identifiability}, content and style from augmentations~\citep{vonkugelgen2021self,zimmermann2021contrastive,daunhawer2023identifiability}, multi-view partial observability~\citep{yao2024multiview,sturma2023unpaired}, sparse cross-modal causal structure in biomedical observations~\citep{sun2025causal}, partial disentanglement for adaptation~\citep{kong2022partial} and sparsity~\citep{lachapelle2022disentanglement}, is unified by viewing identification as preserving known invariances~\citep{yao2025unifying}, and is limited by the impossibility of unsupervised disentanglement without inductive bias~\citep{locatello2019challenging}. Latent structure can remain identifiable under selection bias in linear Gaussian models~\citep{dai2025latent}, a setting close to our enrollment shift but aimed at causal discovery rather than transfer. Our result concerns a function of the latent vector, the evidence about an unobserved state that survives enrollment changes, and its alignment and coordinate premises are explicit. Structural and contextual priors also guide inference beyond medical applications: TreeDiff uses abstract-syntax-tree structure in diffusion-based code generation~\citep{zeng2026treediff}, while Cao et al. combine greedy search and chain-of-thought prompting for context-dependent slang interpretation~\citep{cao2025slang}. These provide examples of task-specific inductive biases; the cross-cohort identification result here rests on the explicit mixture and alignment assumptions of Section~\ref{sec:theory}.

\section{Problem formulation}
\label{sec:setting}

\paragraph*{Data-generating model}
Let $S\in\{0,\dots,K-1\}$ be the latent disease state, $B$ the background factors that pre-exist the disease (age, genetics), and $D$ the physiology that the disease shapes. The physiology is generated by mechanisms that are the same in every cohort, $D_j=f_j(D_{\mathrm{pa}(j)},B,S,\epsilon_j)$ along a directed acyclic graph, so background factors enter the physiology as parents (age enters neurodegeneration). We write $Z=(B,D)$ for the full latent vector and keep $B$ apart from $D$ because a cohort can change the two differently: selection can associate $B$ with $S$, but it cannot alter how $S$ and $B$ shape $D$. Modality $m$ observes a subset of it, $X^{(m)}=g_m(Z_{A_m},\xi_m;\,c_m(e))$, with private variation $\xi_m$; the overlaps of the $A_m$ define the shared, pairwise and private blocks of an \emph{information lattice}. The recorded endpoint $Y\sim p(y\mid s,z;\,c_Y(e))$ is one more partial view, not the state, and can depend on physiology beyond $S$.

\paragraph*{How cohorts differ: selection and measurement}
A cohort $e$ generates its data as
\begin{equation}
\begin{aligned}
p_e(s,b,d,x,y)={}&\underbrace{\pi_e(s)\,p_e(b\mid s)}_{\text{selection}}\;p(d\mid b,s)\\
&\times\prod_m p\big(x^{(m)}\mid z_{A_m};c_m(e)\big)\\
&\times p\big(y\mid s,z;c_Y(e)\big).
\end{aligned}
\label{eq:cohort}
\end{equation}
Only two things depend on $e$. \emph{Selection} is whom the cohort enrolls: its prevalence $\pi_e$ and how strongly its enrollment ties background to disease state, $p_e(b\mid s)$; a memory clinic enrolls cases older than its controls, a population cohort does not. \emph{Measurement}, $c(e)=(c_1(e),\dots,c_M(e),c_Y(e))$, is how the cohort records them: scanners, assays and label protocol. The mechanisms $p(d\mid b,s)$ are fixed. Selection is what corrupts state evidence: the $B$--$S$ association belongs to the sampling law, not to the disease, yet the physiology inherits it, so a pooled learner sees it in every modality that observes $D$. We therefore treat selection as the shift to be removed, hold acquisition fixed for identification, and carry label-protocol drift as a separate term of the outcome risk (Proposition~\ref{prop:decomp}). Orthogonal to both is the \emph{mechanism regime}: additive, where the state shifts physiological levels, or with interactions, where it also changes couplings and part of the evidence lives in covariances.

\paragraph*{Problem}
We observe source cohorts $e\in\Ee_{\mathrm{src}}$ with modalities and endpoint but never $S$, and a target cohort with unlabeled modalities only. We seek a function of the modalities whose evidence about $S$ does not depend on selection and is identifiable without observing $S$, and a predictor of $Y$ that uses it. To state what can be recovered we assume common coordinates $z\in\mathbb R^d$ for the latent vector in which each cohort is a Gaussian mixture over states,
\begin{equation}
p_e(z)=\sum_{s=0}^{K-1}\pi_{es}\,\mathcal N(z;\mu_{es},\Sigma_s),\qquad \pi_{es}>0,\ \Sigma_s\succ0,
\label{eq:mixture}
\end{equation}
with components distinct within each cohort and state-specific covariances shared across cohorts (A1). Cohort-specific weights and means carry selection; shared covariances carry the mechanisms, including interactions. Let $U=\Bb^\top Z$ collect the directions along which the state-conditional densities change across cohorts (the \emph{selection statistic}, constructed in Section~\ref{sec:theory}). For a fixed positive reference prior $\pi$ we define the \emph{conditional disease evidence} and the \emph{reference posterior}
\begin{equation}
\begin{aligned}
T_s(z)&=\log\frac{p_e(z\mid s)}{p_e(z\mid 0)}-\log\frac{p_e(U(z)\mid s)}{p_e(U(z)\mid 0)},\\
\qref(s\mid z)&=\operatorname{softmax}_s\{\log\pi_s+T_s(z)\},
\end{aligned}
\label{eq:target}
\end{equation}
with $T_0=0$. The first term is the evidence a cohort's own state model would use; the second is the part of it that selection created. Their difference keeps the background as a covariate rather than as evidence. In Figure~\ref{fig:motivation}a, where selection acts on the background, the linear-Gaussian instance makes $U$ an invertible linear function of $B$ once source enrollment varies along every background factor (Appendix~\ref{app:structural}), and then $T_s(z)=\log p(d\mid b,s)/p(d\mid b,0)$ is exactly the likelihood ratio of the fixed mechanisms: a laboratory value is read relative to age, while the age--disease association that selection created is removed. The learner need not know which coordinates are background: $U$ is identified from how state densities change across source cohorts. The reference posterior generally differs from the cohort's Bayes posterior $p_e(s\mid z)$, which uses selection as evidence and does not transfer. We further assume (A2) at least three source cohorts that differ in enrollment, (A3) one of the alignment conditions of Proposition~\ref{prop:align}, and, for transfer, (A4) that the target's state-conditional laws are Gaussian with the same $\Sigma_s$ and shifts in $\mathcal U$.

\subsection{Exact examples of selection and joint evidence}
\label{sec:examples}

These examples give exact population errors for specified mechanisms (Figure~\ref{fig:examples}). They motivate the identification analysis in Section~\ref{sec:theory}; they are not fitted \method{} or ERM performance.

\begin{figure}[t]
\centering
\includegraphics[width=0.8\linewidth]{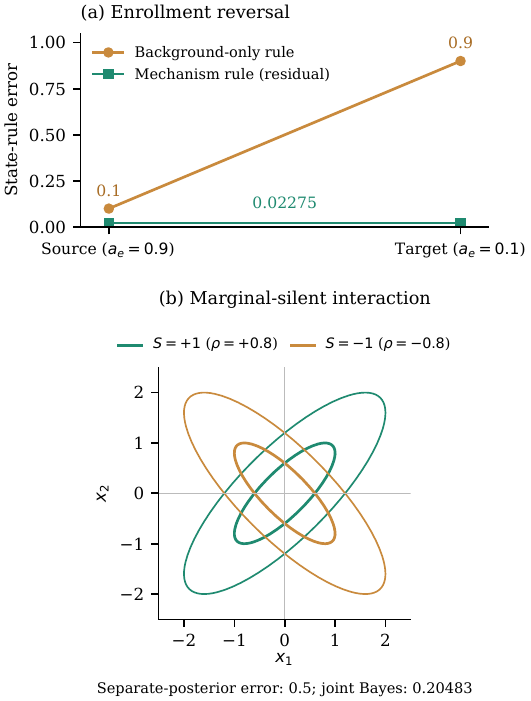}
\caption{\textbf{Analytical illustrations.} (a)~An enrollment reversal changes the error of a background-only rule from $0.1$ to $0.9$, while a rule that uses the specified mechanism through the residual has error $0.02275$ in both cohorts. (b)~Contours at Mahalanobis radii one and two for two Gaussian states with correlations $+0.8$ and $-0.8$. Each single-view posterior is uninformative, whereas the joint interaction yields error $0.20483$. Contours are analytical densities; no error bars apply.}
\label{fig:examples}
\end{figure}

\paragraph*{Enrollment-induced shortcut}
This is the selection of Figure~\ref{fig:motivation}a with one background factor. Let population signs $S,B\in\{-1,+1\}$ be independent and balanced, and let the inclusion indicator $R$ satisfy
\begin{equation}
P(R=1\mid S,B,e)=\begin{cases}a_e,& B=S,\\ 1-a_e,& B\neq S.\end{cases}
\label{eq:inclusion}
\end{equation}
Then $P(R=1)=1/2$, $P(S=+1\mid R=1,e)=1/2$ and $P(B=S\mid R=1,e)=a_e$. The rule $h_B=B$ has selected-cohort error $1-a_e$: $0.1$ for the source law $a_e=0.9$ and $0.9$ for the target law $a_e=0.1$. Let $D=S+B+\epsilon$ with $\epsilon\sim\mathcal N(0,0.5^2)$ independent of $(S,B)$ and of the inclusion noise, inclusion depending only on $(S,B,e)$. With the parents and the mechanism known,
\begin{equation}
\begin{aligned}h_{\mathrm{mech}}&=\operatorname{sign}(D-B),\\ P(h_{\mathrm{mech}}\neq S\mid R=1,e)&=\Phi(-2)\simeq0.02275.\end{aligned}
\end{equation}
Conditioning on $B$ preserves the disease mechanism while removing the selected-cohort association from its evidence, as in the conditional sub-model evidence of Appendix~\ref{app:proofs}. The mechanism and parents are given here; this discrete example lies outside the Gaussian alignment theorem, and target selection can itself carry useful evidence, which is the price in Equation~\eqref{eq:gap}.

\paragraph*{Interaction invisible to separate posteriors}
In a separate two-view model take balanced $S$, independent standard Gaussian noises $\epsilon_1,\epsilon_2$, and
\begin{equation}
\begin{aligned}X_1=D_1&=\epsilon_1,\\ X_2=D_2&=\rho S D_1+\sqrt{1-\rho^2}\,\epsilon_2,\quad \rho=0.8.\end{aligned}
\end{equation}
Both single-view state-conditionals are standard normal, so each single-view Bayes posterior equals $1/2$, and any rule that uses only those two posterior values has error $1/2$. In contrast,
\begin{equation}
\begin{aligned}\log\frac{p(x_1,x_2\mid S=+1)}{p(x_1,x_2\mid S=-1)}&=\frac{2\rho x_1x_2}{1-\rho^2},\\ R^*_{\mathrm{joint}}&=\frac{\arccos\rho}{\pi}\simeq0.20483.\end{aligned}
\end{equation}
The Bayes rule is $\operatorname{sign}(X_1X_2)$. The loss concerns separate-posterior compression, not all late fusion; in Figure~\ref{fig:motivation}a the laboratory modality $X^{(2)}$ observes both $D_1$ and $D_2$, so this particular bottleneck does not apply there.

\section{Identifying disease evidence without observing the state}
\label{sec:theory}

We first align the latent states across cohorts, then identify the selection statistic and the evidence that remains after conditioning on it, show why estimating this evidence requires joint observations and information-preserving representations, and relate it to prediction of the recorded endpoint. The identification results assume the common coordinates of Equation~\eqref{eq:mixture}; recovering such coordinates from raw modalities is a separate requirement. Full statements and proofs appear in Appendix~\ref{app:proofs}.

\subsection{Aligning unknown states across cohorts}
Each cohort's Gaussian mixture identifies its components only up to a permutation~\citep{yakowitz1968identifiability}, and independent permutations in different cohorts do not define a common disease variable. Either of the following conditions establishes the correspondence.

\begin{proposition}[Global state alignment]\label{prop:align}
Under Equation~\eqref{eq:mixture}, states can be aligned across source cohorts up to one global permutation if the shared covariances $\{\Sigma_s\}$ are pairwise distinct. For the three-state family with a common covariance, alignment also holds, up to a global exchange of the two end states, if
\begin{equation}
\mu_{e0}=-\kappa_e\nu,\qquad
\mu_{e1}=a,\qquad
\mu_{e2}=2a+\kappa_e\nu,
\label{eq:parallel}
\end{equation}
where $a$ and $\nu$ are linearly independent and at least three source cohorts have distinct values of $\kappa_e$.
\end{proposition}

Distinct covariances provide a unique match for each state. In Equation~\eqref{eq:parallel}, $a$ is both a component mean and the arithmetic centre of every cohort's three component means. After centring at $a$, the remaining means lie on two distinct parallel lines, each containing at least three observed points. These lines determine the correspondence up to their exchange, without knowledge of $\kappa_e$ or a zero-enrollment cohort. Alignment assigns consistent state identities, not their clinical interpretation.

\subsection{The selection statistic and identifiable conditional evidence}
Once states are aligned, their changes across cohorts identify which statistic to condition on. For a reference source cohort $e_0$, define
\begin{equation}
\mathcal U=\operatorname{span}\!\left\{
\Sigma_s^{-1}(\mu_{es}-\mu_{e_0s}):e,s\right\}.
\label{eq:B}
\end{equation}
where $\Gamma$ is a full-column-rank basis of $\mathcal U$. These natural-parameter changes differ from mean differences when physiology is correlated.

\begin{theorem}[Identifiable conditional evidence]\label{thm:evidence}
Assume Equation~\eqref{eq:mixture} and an alignment condition of Proposition~\ref{prop:align}. In the common coordinates, $\mathcal U$ is identifiable from the source laws and is the smallest linear space whose associated statistic determines every same-state density ratio between source cohorts. With $U=\Gamma^\top Z$, the conditional law of $Z$ given $(U,S)$ is common to all source cohorts. Consequently, the reference posterior in Equation~\eqref{eq:target} is identifiable up to one global state permutation, with $\pi$ permuted accordingly; the evidence $T$ is identified with the corresponding choice of baseline state. These conclusions extend to targets satisfying (A4).
\end{theorem}

The key identity is that the same-state log-density ratio between two cohorts is linear in $z$ with coefficient in $\mathcal U$, so it depends only on $U$ and conditioning on $U$ cancels it. If $V$ completes $U$ to invertible coordinates, the remaining conditional density is a shared kernel $h_s(v\mid u)$, giving $T_s(z)=\log h_s(v\mid u)-\log h_0(v\mid u)$ for any basis of $\mathcal U$. Useful evidence requires variation beyond the selection statistic: if $\mathcal U=\mathbb R^d$, then $T=0$ and $\qref=\pi$. Source variation must also cover the target shifts stipulated by (A4), and retaining fewer directions than $\dim\mathcal U$, as a low-rank estimator may do, is an approximation. Appendix~\ref{app:structural} gives sufficient structural conditions for rank one, which a single scalar enrollment parameter does not by itself ensure.

\subsection{Reading the modalities jointly and without loss}
Evidence can be lost in two ways. First, fusing per-modality state predictions discards cross-modal dependence, which can carry all state information: if two Gaussian views have correlation $+\rho$ in one state and $-\rho$ in the other, both single-view posteriors are constant, whereas the joint Bayes error is $\arccos(\rho)/\pi<1/2$ (Proposition~\ref{prop:interaction}, Appendix~\ref{app:interaction}). The state model therefore retains covariance terms. Second, the encoder must preserve the information the evidence depends on. Write $\tilde Z=(Z,\xi_{1:M})$ and call the coordinates observed by exactly the modalities in $a$ the block of $a$.

\begin{proposition}[Evidence-preserving representations]\label{prop:lattice}
Fix the acquisition context. Suppose the view maps are diffeomorphisms onto their images, every coordinate of $Z$ is observed by at least one modality, and $\tilde Z$ has a smooth positive density on an open, simply connected product domain.
(i)~Under the known-dimension and population optimization conditions of \citet[Thm.~3.2]{yao2024multiview}, the content shared by a subset of views is recoverable up to a diffeomorphism. View geometry alone permits an exact-support block to be reparameterized conditionally on blocks observed by strict supersets.
(ii)~A common diffeomorphism $F$ of $Z$ preserves $T$ when the conditioning statistic is transported as $U\circ F^{-1}$. If an encoder permits exact reconstruction of all modalities, then $T$ and $\qref$ are functions of its output. Retaining only content shared by all modalities can lose it.
\end{proposition}

Shared content for a subset includes the blocks shared by larger subsets. This motivates organizing the representation from broadly shared to partially shared and private information. Reconstruction supplies a sufficient condition for preserving evidence that alignment alone may omit. The population statement does not imply that finite reconstruction error, or a particular residual objective, identifies the exact blocks. Moreover, preserving evidence does not ensure that it is expressed in the Gaussian coordinates needed by our estimator (Appendix~\ref{app:lattice}).

\subsection{From disease evidence to a recorded endpoint}
The reference posterior excludes the target's selection gate $g_s(u)=P_e(S=s\mid U=u)$ and generally differs from its Bayes posterior; for a target satisfying (A4), the excess state log-loss satisfies
\begin{equation}
\begin{aligned}
R^S_e(\qref)-R^S_e(p_e)
&\leq \mathbb E_U\,\mathrm{KL}\bigl(g(U)\,\|\,\pi\bigr)\\
&= I_e(S;U)+\mathrm{KL}\bigl(P_e(S)\,\|\,\pi\bigr).
\end{aligned}
\label{eq:gap}
\end{equation}
The bound vanishes when $S$ and $U$ are independent and $\pi$ matches target prevalence; Appendix~\ref{app:gap} gives the exact discrepancy, which need not increase monotonically with an enrollment parameter. For observed inputs, assume $Z=\Phi(X)$ in the fixed acquisition context and write $q^{X,\mathrm{ref}}_\pi(x)=\qref(\cdot\mid\Phi(x))$, $p^X_e(x)=P_e(S=\cdot\mid X=x)$ and $\eta_e(x)=P_e(Y=1\mid X=x)$.

\begin{proposition}[State-to-task decomposition]\label{prop:decomp}
Suppose, pointwise in $x$, that $|\eta_e-h(p^X_e,r_e)|\leq d_e$ for continuous task information $r_e$ and a map $h$ shared across cohorts and $(L_S,L_R)$-Lipschitz in its two arguments. Let $\hat f=\hat h(\hat q,\hat r)$ be a head with approximation error $\epsilon_h=|\hat h(\hat q,\hat r)-h(\hat q,\hat r)|$. Then
\[
\begin{aligned}
|\hat f-\eta_e|\leq A_e:={}&\epsilon_h+L_S\|\hat q-q^{X,\mathrm{ref}}_\pi\|_1\\
&+L_S\|q^{X,\mathrm{ref}}_\pi-p^X_e\|_1\\
&+L_R\|\hat r-r_e\|+d_e,
\end{aligned}
\]
and the excess binary Brier risk over $\eta_e$ is at most $\mathbb E_e[A_e(X)^2]$.
\end{proposition}

The terms separate head approximation, estimation of the reference posterior, its discrepancy from the target posterior (linked to Equation~\eqref{eq:gap} by Pinsker's inequality; Appendix~\ref{app:decomp}), continuous task information, and residual endpoint mismatch, including label-protocol drift. Better state evidence directly reduces the second term but need not remove the others, so large state gains can coexist with small outcome gains; this motivates the continuous path to the outcome in Section~\ref{sec:method}, where the decomposition describes a coupled approximation problem rather than separately trained stages.

\section{DISCERN: Jointly Learning Conditional Disease Evidence}
\label{sec:method}

\method{} (DISease-state Content idEntification for Robust multimodal learNing) combines multimodal representation learning, latent-state modeling, and outcome prediction in a jointly optimized framework (Figure~\ref{fig:discern_joint}). Section~\ref{sec:theory} specifies the evidence to estimate; the learning objective couples this estimate to a representation that preserves multimodal information and supports the recorded endpoint. Disease-state labels are never required: recorded outcomes supervise the prediction head and, through joint optimization, shape the representation and state model.

\begin{figure*}[t]
\centering
\includegraphics[width=\linewidth]
{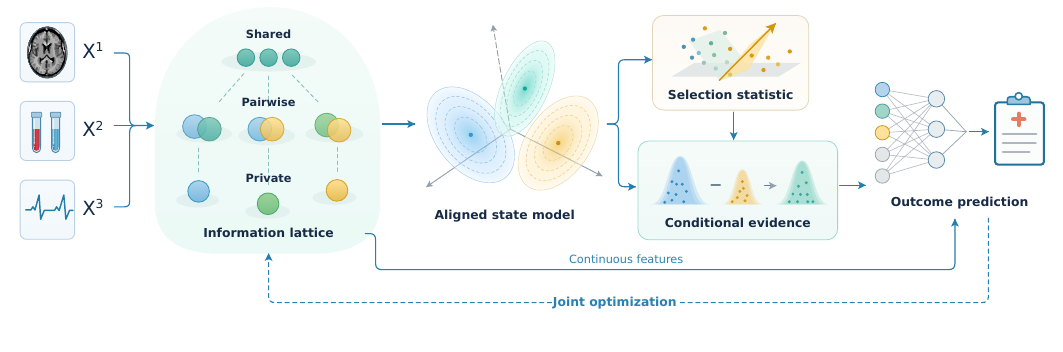}
\caption{\textbf{Joint learning of conditional disease evidence.}
An information-lattice encoder splits the modalities into shared, pairwise, and private blocks; a source-cohort state model supplies the natural-parameter shifts that define the selection statistic; conditional evidence (joint minus marginal log-likelihood ratio) and continuous features predict the outcome; all parts are trained jointly.}
\label{fig:discern_joint}
\end{figure*}

\subsection{Representing overlapping multimodal information}
Let $E_\phi$ be an information-lattice encoder with output $\hat z=E_\phi(X)$. Each block $\hat z_a$ corresponds to information available to the modalities indexed by $a\subseteq[M]$. For three modalities, the lattice includes one shared block, three pairwise blocks, and three private blocks. Blocks are extracted from coarse to fine: the shared block is estimated first, pairwise blocks explain the residuals $\delta_m=\bar x_m-G_m^{\mathrm{sh}}(\hat z_{\mathrm{sh}})$, and private blocks represent the remaining variation. The information objective is
\begin{equation}
\begin{aligned}
\mathcal L_{\mathrm{info}}={}&\sum_{|a|\geq2}\mathcal L^{\mathrm{align}}_a
+\sum_m\|\bar x_m-G_m(\hat z_{a\ni m})\|^2\\
&+\sum_m\bigl[\|\delta_m\|^2
+\|\delta_m-G_m^{\mathrm{pw}}(\hat z_{|a|=2,\,a\ni m})\|^2\bigr]\\
&+\lambda_{\mathrm{var}}\mathcal L_{\mathrm{var}}.
\end{aligned}
\label{eq:lattice}
\end{equation}
Alignment encourages consistent estimates of shared blocks within each source cohort, reconstruction encourages retention of information beyond those blocks, and the residual terms organize the partially shared information. The variance penalty discourages collapse. Reconstruction is motivated by Proposition~\ref{prop:lattice}, with the caveats noted there; it stays active during joint training.

\subsection{Modeling latent states across source cohorts}
We use a Gaussian state model
\begin{equation}
p_\psi(\hat z\mid e)
=\sum_{s=0}^{K-1}\pi_{es}\,
\mathcal{N}(\hat z;\mu_{es},\Sigma_s),
\label{eq:joint_state_model}
\end{equation}
where cohort-specific means and mixture weights accommodate selection, while the covariance $\Sigma_s$ of each state is shared across cohorts. In the clinical implementation, a block-ordered Gaussian structural equation model parameterizes the covariances through state-dependent slopes. EM on source features initializes the state model, with consistent state correspondence across cohorts motivated by Proposition~\ref{prop:align}; the model then enters joint optimization through the feature-density regularizer
\begin{equation}
\mathcal{L}_{\mathrm{dens}}
=-\mathbb{E}_{(X,e)\sim\mathcal{D}_{\mathrm{src}}}
\log p_\psi(E_\phi(X)\mid e).
\label{eq:joint_density}
\end{equation}
It encourages coordinates compatible with the state model; lacking a raw-input change-of-variables term, it is a feature regularizer rather than a normalized generative likelihood of the modalities.

\subsection{Estimating the selection statistic and conditional evidence}
At the current state-model parameters, let $\bar\mu_s=\sum_e w_e\mu_{es}$, where $w_e$ are normalized source weights. We form the natural-parameter shift matrix with columns $\Sigma_s^{-1}(\mu_{es}-\bar\mu_s)$ and obtain a basis $\Gamma$ from its leading singular directions. The experiments retain one direction. The resulting statistic $U=\Gamma^\top\hat z$ estimates the source selection variation characterized by Theorem~\ref{thm:evidence}. Using the reference family $\mathcal{N}(\bar\mu_s,\Sigma_s)$, define $\hat\ell_s(A)=\log\hat p(A\mid s)-\log\hat p(A\mid 0)$. The conditional evidence and reference posterior are
\begin{equation}
\begin{aligned}
\hat T_s(\hat z)
&=\hat\ell_s(\hat z)-\hat\ell_s(\Gamma^\top\hat z),\\
\hat q_\pi(s\mid\hat z)
&=\operatorname{softmax}_s\!
\left(\log\pi_s+\hat T_s(\hat z)\right),
\end{aligned}
\label{eq:joint_evidence}
\end{equation}
where $\pi$ is a fixed positive reference prior. Subtracting the marginal evidence removes the contribution associated with the estimated selection statistic while retaining the conditional relations among physiological quantities, including state-dependent interactions through the covariance terms. The basis $\Gamma$ is recomputed once per epoch outside gradient tracking and held fixed in between.

\paragraph*{Closed-form Gaussian correction}
The estimator uses $\bar\mu_s=\sum_e w_e\mu_{es}$ with $w_e\geq0$ and $\sum_e w_e=1$. This does not change the population subspace:
\begin{equation}
\begin{aligned}
&\operatorname{span}\!\left\{\Sigma_s^{-1}(\mu_{es}-\bar\mu_s):e,s\right\}\\
&\quad=\operatorname{span}\!\left\{\Sigma_s^{-1}(\mu_{es}-\mu_{e_0s}):e,s\right\}
=\mathcal U.
\end{aligned}
\label{eq:centered_span}
\end{equation}
Each centred contrast is a linear combination of reference-cohort contrasts, and subtracting the centred contrast for $e_0$ gives the converse. Moreover, $\Sigma_s^{-1}(\bar\mu_s-\mu_{e_0s})\in\mathcal U$, so the Gaussian reference family $\mathcal N(\bar\mu_s,\Sigma_s)$ has the same conditional kernels $h_s$. It is a single Gaussian for each state, not the pooled mixture of that state's source densities. For $0<r\leq d$, its induced marginal on $U$ gives
\begin{equation}
\begin{aligned}
T_s(z)={}&\log\frac{\mathcal N_d(z;\bar\mu_s,\Sigma_s)}
{\mathcal N_d(z;\bar\mu_0,\Sigma_0)}\\
&-\log\frac{\mathcal N_r(u;\Gamma^\top\bar\mu_s,\Gamma^\top\Sigma_s\Gamma)}
{\mathcal N_r(u;\Gamma^\top\bar\mu_0,\Gamma^\top\Sigma_0\Gamma)},
\end{aligned}
\label{eq:gaussian_evidence}
\end{equation}
Here $u=\Gamma^\top z$. The covariance in each marginal is $\Gamma^\top\Sigma_s\Gamma$; it is positive definite for full-column-rank $\Gamma$. The joint and marginal terms must come from the same state-conditional reference family. Equation~\eqref{eq:joint_evidence} is this construction with learned parameters and coordinates. Exact invariance requires the full subspace; truncating it to a smaller rank generally leaves some selection dependence. Periodically refreshing the basis during joint training changes the estimate, not the population identity.

\subsection{Joint optimization with the recorded outcome}
The outcome predictor combines the conditional state log odds with continuous features $r_\phi(X)$:
\begin{equation}
\begin{aligned}
&\hat p_{\phi,\psi,\theta}(Y=1\mid X)\\
&\quad=h_\theta\!\left(
r_\phi(X),
\left\{\log\frac{\hat q_\pi(s\mid E_\phi(X))}
{\hat q_\pi(0\mid E_\phi(X))}\right\}_{s=1}^{K-1}
\right).
\end{aligned}
\label{eq:joint_outcome}
\end{equation}
The continuous path retains task information beyond the discrete state, as motivated by Proposition~\ref{prop:decomp}; in the clinical model it also includes the context branch described in Appendix~\ref{app:icu}. With binary cross-entropy denoted by $\mathcal{L}_{\mathrm{task}}$, training minimizes
\begin{equation}
\min_{\phi,\psi,\theta}\quad
\mathcal{L}_{\mathrm{task}}
+\lambda_{\mathrm{info}}\mathcal{L}_{\mathrm{info}}
+\lambda_{\mathrm{dens}}\mathcal{L}_{\mathrm{dens}},
\label{eq:discern_joint_objective}
\end{equation}
including the reconstruction decoders among the representation parameters. Outcome gradients pass through both the continuous and conditional-evidence paths, coupling the encoder, state model, and head. The information and density terms act as regularizers ($\lambda_{\mathrm{info}}=0.01$ and $\lambda_{\mathrm{dens}}=0.001$ in the clinical experiments). Joint training requires no observed $S$, although the recorded $Y$ informs the learned states through these gradients. Training starts from a supervised warm-up of the backbone and EM initialization of the state model, with model selection on source data only (Appendix~\ref{app:icu}); at deployment all components are frozen and use only the target patient's inputs.

\section{Experiments}
\label{sec:experiments}

We organize the evaluation around three questions: whether conditioning removes enrollment-dependent state evidence when the required coordinates are available (Theorem~\ref{thm:evidence}); whether learned representations supply those coordinates or merely preserve state information (Proposition~\ref{prop:lattice}); and whether the resulting evidence helps predict recorded clinical endpoints (Proposition~\ref{prop:decomp}). The synthetic analyses isolate the first two questions and their implications for outcome prediction; the real-world analyses evaluate the jointly trained clinical model and its matched ablations.

\subsection{Datasets}
\label{sec:exp-data}
\paragraph*{Synthetic benchmark}
One structural model with three states and ten latent factors generates three partially overlapping modalities through leaky-affine views with private variation. Twelve source cohorts (2{,}000 training and 400 validation observations each) differ in how strongly enrollment ties the root background factors to the state, as in Figure~\ref{fig:motivation}a, and in label style. Four target laws keep, remove, reverse or strengthen the training enrollment. The generator also changes label style and site, so the target laws combine selection and measurement shifts, with enrollment as the principal axis. Two mechanism regimes, additive and with interactions, each have five generated studies and 10{,}000 test observations per target law. The recorded endpoint $Y$ depends on continuous factors and cohort-specific label terms, the $c_Y(e)$ of Equation~\eqref{eq:cohort}. A separate three-factor generator supports the factor-count sensitivity analysis; its absolute results are not pooled with the ten-factor benchmark.

\paragraph*{ICU cohorts}
We evaluate same-admission hospital mortality between eICU-CRD~2.0~\citep{pollard2018eicu} and MIMIC-IV~2.0~\citep{johnson2023mimic} in both directions, using 48 hours of observations, a 6-hour gap and patient-disjoint partitions. Eligible patients are at least 18 years old, with ICU stays of 54--240 hours. The fixed MIMIC-IV target contains 9{,}922 stays from 9{,}922 patients at one physical hospital; the eICU target contains 26{,}763 stays from 22{,}982 patients at 57 hospitals. The databases differ in enrollment, charting and laboratory practice, so acquisition is not held fixed. Target adaptation partitions remain unused. Appendix~\ref{app:icu} reports source role sizes, exclusions and record-availability limitations.

\paragraph*{Exploratory ADNI extension}
An earlier categorical variant classifies contemporaneous CN, MCI and AD diagnoses from T1 MRI, demographics, behaviour and genetics/history. Among 1{,}807 patients, 1{,}264 form source training, 272 source validation and 271 an age-shift target defined by the oldest 15\% within each diagnosis. Reserving separate 54-patient validation and target context sets leaves 218 patients for downstream checkpoint selection and a shared 217-patient primary target. This variant uses a different estimator and is interpreted separately (Appendix~\ref{app:adni}).

\subsection{Implementation details}
\label{sec:exp-implementation}
\paragraph*{Synthetic estimators and controls}
The common-coordinate experiments supply the generator's ten-dimensional $Z$; $S$ is hidden during fitting and used only for evaluation. The state mixture uses $K=3$, at most 150 EM iterations, two restarts and covariance ridge $10^{-3}$. Source-validation density selects common or state-specific covariance; aligned natural-parameter shifts provide a rank-one correction. The averaged estimator combines the original fit and two within-cohort bootstrap fits with equal weights. A separate attribution study compares uncorrected, random-direction and estimated-direction readouts of the same fitted mixture. Outcome heads use two width-64 hidden layers; matched ERM zeros the evidence slots, and the state-only control zeros the continuous coordinates. AdamW uses learning rate $10^{-3}$, weight decay $10^{-4}$, batch size 512, at most 40 epochs and patience 6 on source cohort-average NLL.

The learned-view study replaces $Z$ with a frozen 13-dimensional lattice with block dimensions $(2,1,1,1,3,2,3)$, trained without $S$ or $Y$ using the information objective. It uses three paired studies per regime. Mixture-only, affine and nonlinear invertible refinements test whether a better density fit supplies the required coordinates. Post-fit target-labeled probes diagnose retained state information and never update or select deployed models. Appendix~\ref{app:synthetic} gives seeds, encoder training, alignment, calibration and probe protocols.

\paragraph*{Clinical joint training and matched ablations}
All five rich ICU variants share a 38-dimensional physiology lattice with block dimensions $(4,2,2,2,4,8,16)$, a separate context branch, prediction inputs and source-only model selection. A five-epoch supervised warm-up with AdamW at $10^{-3}$ precedes at most 30 joint epochs at $3\times10^{-4}$, weight decay $10^{-4}$ and batch size 256, with patience 6 on source-validation group-average AUROC. The information and mechanism-density losses have weights $0.01$ and $0.001$; the rank-one selection basis is updated once per epoch. Three seeds (42, 3407 and 2026) repeat training. Because outcome gradients reach the encoder, clinical states are label-informed. Historical baselines share the target cohorts but use different backbones and schedules; Appendix~\ref{app:icu} specifies these differences. The ADNI extension instead reuses one source-supervised encoder and repeats downstream fitting over the same three seeds, as detailed in Appendix~\ref{app:adni}.

\paragraph*{Evaluation and model selection}
State AUROC averages the binary AUROCs for $S>0$ and $S>1$ using posterior mean severity; source $S$ fixes a single global scoring permutation after fitting. Outcome metrics include AUROC, AUPRC, NLL and Brier score. ICU AUROC and AUPRC average physical hospitals, whereas NLL pools target patients. Tables report mean and sample SD across generated studies or training seeds, as specified. Hyperparameters and checkpoints use source data throughout. Some clinical target cohorts were inspected in earlier development rounds, so the clinical comparisons remain exploratory.

\subsection{Analysis on synthetic datasets}
\label{sec:exp-synthetic}
The common-coordinate, learned-view and factor-count studies test different parts of the identification argument. Their controls and replication units are kept separate below.

\subsubsection{State recovery across enrollment laws}
\label{sec:exp-state}
Table~\ref{tab:state} and Figure~\ref{fig:motivation}b test Theorem~\ref{thm:evidence} on common coordinates, comparing the uncorrected state model, the conditional evidence of \method{} (three aligned fits averaged) and a reference oracle that reads the true mechanism with a fixed population prior. Reversal exposes the shortcut: the uncorrected model falls to $0.043$ and $0.265$ state AUROC, and conditioning raises it to $0.883$ and $0.825$. In the additive regime, conditional evidence is nearly flat across the four laws; with interactions, it ranges from $0.825$ to $0.870$. It tracks the reference oracle, with the largest gap ($0.031$) under reversal with interactions. The fourth law shows the price in Equation~\eqref{eq:gap}: a target that selects strongly makes selection itself evidence, which the uncorrected readout exploits ($0.995$, $0.984$) and the conditional evidence removes ($0.880$, $0.870$).

\begin{table*}[t]
\centering\small
\caption{\textbf{State recovery across enrollment laws} (state AUROC on common coordinates, mean$\pm$SD over five generated studies). Targets keep, remove, reverse or strengthen the training enrollment. Bold: better of the two deployable estimators; the reference oracle reads the true mechanism.}
\label{tab:state}
\resizebox{\linewidth}{!}{%
\begin{tabular}{llcccc}
\toprule
Regime & State estimator & same & off & reversed & stronger\\
\midrule
additive & uncorrected latent-state model & $0.839{\scriptstyle\pm.318}$ & $0.771{\scriptstyle\pm.042}$ & $0.043{\scriptstyle\pm.048}$ & $\mathbf{0.995}{\scriptstyle\pm.007}$\\
additive & \textbf{\method{} conditional evidence} & $\mathbf{0.884}{\scriptstyle\pm.048}$ & $\mathbf{0.880}{\scriptstyle\pm.050}$ & $\mathbf{0.883}{\scriptstyle\pm.053}$ & $0.880{\scriptstyle\pm.049}$\\
additive & reference oracle & $0.882{\scriptstyle\pm.050}$ & $0.882{\scriptstyle\pm.050}$ & $0.882{\scriptstyle\pm.050}$ & $0.882{\scriptstyle\pm.050}$\\
\midrule
with interactions & uncorrected latent-state model & $\mathbf{0.939}{\scriptstyle\pm.089}$ & $0.769{\scriptstyle\pm.073}$ & $0.265{\scriptstyle\pm.139}$ & $\mathbf{0.984}{\scriptstyle\pm.012}$\\
with interactions & \textbf{\method{} conditional evidence} & $0.869{\scriptstyle\pm.056}$ & $\mathbf{0.849}{\scriptstyle\pm.047}$ & $\mathbf{0.825}{\scriptstyle\pm.067}$ & $0.870{\scriptstyle\pm.059}$\\
with interactions & reference oracle & $0.854{\scriptstyle\pm.049}$ & $0.855{\scriptstyle\pm.046}$ & $0.856{\scriptstyle\pm.045}$ & $0.853{\scriptstyle\pm.050}$\\
\bottomrule
\end{tabular}}
\end{table*}

\subsubsection{Ablation of the conditioning direction}
A gain under reversal could come from any change to the state model. Five further studies per regime read one fitted mixture in three ways: uncorrected, conditioned on a random unit direction, and conditioned on the estimated direction (Figure~\ref{fig:attrib}, Appendix~\ref{app:synthetic}). Under reversal the estimated direction reaches $0.860$ and $0.861$ state AUROC, against $0.151$ and $0.270$ for a random direction and $0.115$ and $0.252$ uncorrected, improving over both in every study. The correction is attributable to the selection statistic of Theorem~\ref{thm:evidence}, not to conditioning as such.

\begin{figure}[t]
\centering
\includegraphics[width=\linewidth]{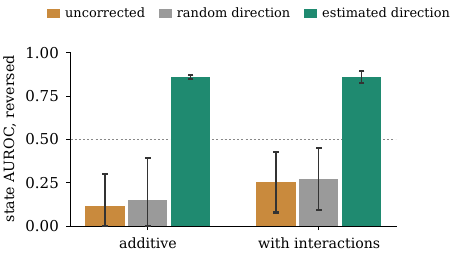}
\caption{\textbf{Attribution of the estimated direction.} State AUROC under reversed enrollment for three readouts of the same fitted mixture: uncorrected, conditioned on a random unit direction, and conditioned on the estimated selection direction. The comparison tests whether the gain comes from the estimated direction; it is not an ablation of the number of latent states. Bars show means; error bars indicate sample standard deviations across five independently generated studies (truncated at 0 where one SD exceeds the mean).}
\label{fig:attrib}
\end{figure}
\subsubsection{From state evidence to the recorded endpoint}
\label{sec:exp-outcome}
By Proposition~\ref{prop:decomp}, large state gains need not carry over to $Y$. On common coordinates under reversal, the head reading continuous $Z$ with the averaged conditional evidence reaches $0.780$ and $0.812$ outcome AUROC against $0.779$ and $0.810$ for matched ERM, and a head that reads only the state falls to $0.697$ and $0.744$ (Table~\ref{tab:outcome}). Physiology carries task information that no state does.

\begin{table*}[t]
\centering\small
\caption{Recorded-endpoint prediction under reversed enrollment, given common $Z$: mean$\pm$SD over five generated studies.}
\label{tab:outcome}
\begin{tabular}{llcccc}
\toprule
Regime & Head & AUROC & AUPRC & NLL & Brier\\
\midrule
additive & ERM & $0.779{\scriptstyle\pm.019}$ & $0.878{\scriptstyle\pm.042}$ & $0.538{\scriptstyle\pm.053}$ & $0.181{\scriptstyle\pm.019}$\\
additive & + conditional average & $0.780{\scriptstyle\pm.021}$ & $0.878{\scriptstyle\pm.040}$ & $0.532{\scriptstyle\pm.044}$ & $0.179{\scriptstyle\pm.016}$\\
additive & + conditional + style & $0.778{\scriptstyle\pm.013}$ & $0.877{\scriptstyle\pm.040}$ & $0.529{\scriptstyle\pm.051}$ & $0.177{\scriptstyle\pm.020}$\\
additive & state only & $0.697{\scriptstyle\pm.066}$ & $0.824{\scriptstyle\pm.062}$ & $0.602{\scriptstyle\pm.091}$ & $0.204{\scriptstyle\pm.031}$\\
\midrule
with interactions & ERM & $0.810{\scriptstyle\pm.073}$ & $0.884{\scriptstyle\pm.087}$ & $0.515{\scriptstyle\pm.145}$ & $0.168{\scriptstyle\pm.051}$\\
with interactions & + conditional average & $0.812{\scriptstyle\pm.068}$ & $0.887{\scriptstyle\pm.084}$ & $0.513{\scriptstyle\pm.143}$ & $0.168{\scriptstyle\pm.050}$\\
with interactions & + conditional + style & $0.810{\scriptstyle\pm.071}$ & $0.884{\scriptstyle\pm.088}$ & $0.513{\scriptstyle\pm.147}$ & $0.168{\scriptstyle\pm.052}$\\
with interactions & state only & $0.744{\scriptstyle\pm.047}$ & $0.852{\scriptstyle\pm.098}$ & $0.566{\scriptstyle\pm.119}$ & $0.192{\scriptstyle\pm.045}$\\
\bottomrule
\end{tabular}
\end{table*}
\subsubsection{Learned coordinates and state readability}
Learning the coordinates is the harder step. On three paired studies a label-free 13-dimensional lattice replaces the supplied $Z$, and state recovery fails: under reversal its state AUROC is $0.118$ and $0.306$, below chance like the uncorrected model on common coordinates, and an invertible calibration fitted jointly with the mixture (Appendix~\ref{app:learned}) leaves it at $0.135$ and $0.307$ although it improves source feature NLL from $5.24$ to $4.23$ and from $5.93$ to $4.94$. The outcome head matches but does not improve on the matched lattice ERM head, $0.850$ and $0.797$ against $0.863$ and $0.806$ AUROC, within one SD (Table~\ref{tab:learned}). The information is not gone: diagnostic probes trained on separate target-labeled samples read $S$ from the frozen lattice at $0.868$ and $0.832$ under the selection-off law, close to $0.903$ and $0.859$ on the true $Z$. Consistent with Proposition~\ref{prop:lattice}, the reconstruction-trained lattice keeps the evidence but not necessarily in coordinates satisfying Equation~\eqref{eq:mixture}; target views also change site, which the proposition holds fixed, so the two causes are confounded.

\begin{table*}[t]
\centering\footnotesize
\setlength{\tabcolsep}{3pt}
\caption{Learned-view interface under reversed enrollment: the same frozen 13-dimensional lattice is used by all heads; \method{} variants add state evidence. Mean$\pm$SD over three paired studies. Lattice + ERM outputs no state.}
\label{tab:learned}
\begin{tabular*}{\textwidth}{@{\extracolsep{\fill}}lcccccc@{}}
\toprule
& \multicolumn{2}{c}{State AUROC} & \multicolumn{2}{c}{Outcome AUROC} & \multicolumn{2}{c}{Outcome NLL}\\
\cmidrule(lr){2-3}\cmidrule(lr){4-5}\cmidrule(lr){6-7}
Method & Additive & Interactions & Additive & Interactions & Additive & Interactions\\
\midrule
Lattice + ERM & -- & -- & $0.863{\scriptstyle\pm.029}$ & $0.806{\scriptstyle\pm.104}$ & $0.507{\scriptstyle\pm.085}$ & $0.622{\scriptstyle\pm.262}$\\
\method{} base & $0.118{\scriptstyle\pm.095}$ & $0.306{\scriptstyle\pm.083}$ & $0.850{\scriptstyle\pm.034}$ & $0.797{\scriptstyle\pm.107}$ & $0.502{\scriptstyle\pm.092}$ & $0.669{\scriptstyle\pm.295}$\\
\method{} mixture refinement & $0.125{\scriptstyle\pm.076}$ & $0.276{\scriptstyle\pm.095}$ & $0.849{\scriptstyle\pm.032}$ & $0.798{\scriptstyle\pm.104}$ & $0.514{\scriptstyle\pm.103}$ & $0.678{\scriptstyle\pm.310}$\\
\method{} affine calibration & $0.130{\scriptstyle\pm.080}$ & $0.279{\scriptstyle\pm.099}$ & $0.850{\scriptstyle\pm.023}$ & $0.796{\scriptstyle\pm.107}$ & $0.495{\scriptstyle\pm.063}$ & $0.673{\scriptstyle\pm.303}$\\
\method{} nonlinear calibration & $0.135{\scriptstyle\pm.084}$ & $0.307{\scriptstyle\pm.077}$ & $0.849{\scriptstyle\pm.021}$ & $0.795{\scriptstyle\pm.105}$ & $0.522{\scriptstyle\pm.080}$ & $0.683{\scriptstyle\pm.311}$\\
\bottomrule
\end{tabular*}
\end{table*}
\subsubsection{Sensitivity to the number of latent factors}
\label{sec:exp-nfactors}
The structural state model assumes a number $N$ of latent factor nodes, which clinical data do not reveal. Figure~\ref{fig:nfactors} varies $N\in\{2,3,4,5\}$. Panels (a) and (b) use a separate generator that follows the structural model of Figure~\ref{fig:motivation}a with three states and three latent physiology factors, so that $N=3$ is correct, in the additive and interaction regimes; the coordinates are supplied as in Section~\ref{sec:exp-state}, the ERM reference is a transformer on the same coordinates, and each point averages three seeds. Panel (c) reports the corresponding clinical sensitivity analysis discussed in Section~\ref{sec:exp-icu-factors}.

State recovery depends on the assumed number. State AUROC peaks at the true $N=3$ in both regimes and falls to about chance or below when $N$ is misspecified, lowest at $N=2$ and $N=5$ in the additive regime. Even at $N=3$ it is weak, about $0.54$ (additive) and $0.64$ (with interactions), well below Table~\ref{tab:state}; the generator and the state model differ from that benchmark, so the levels are not comparable. The outcome gain is less sensitive: its mean is positive at every $N$ in both regimes, between about $3$ and $5$ points, but the additive error bars include zero at $N=4$ and $N=5$, and the additive gain is largest at $N=2$, where state AUROC is below chance. This is the pattern of Proposition~\ref{prop:decomp}: the head also reads the continuous coordinates, and state recovery improves only one term of the bound. Because the reference is a transformer rather than the matched head of Table~\ref{tab:outcome}, these gains are not comparable with the near-zero gains there.

\begin{figure*}[t]
\centering
\includegraphics[width=\linewidth]{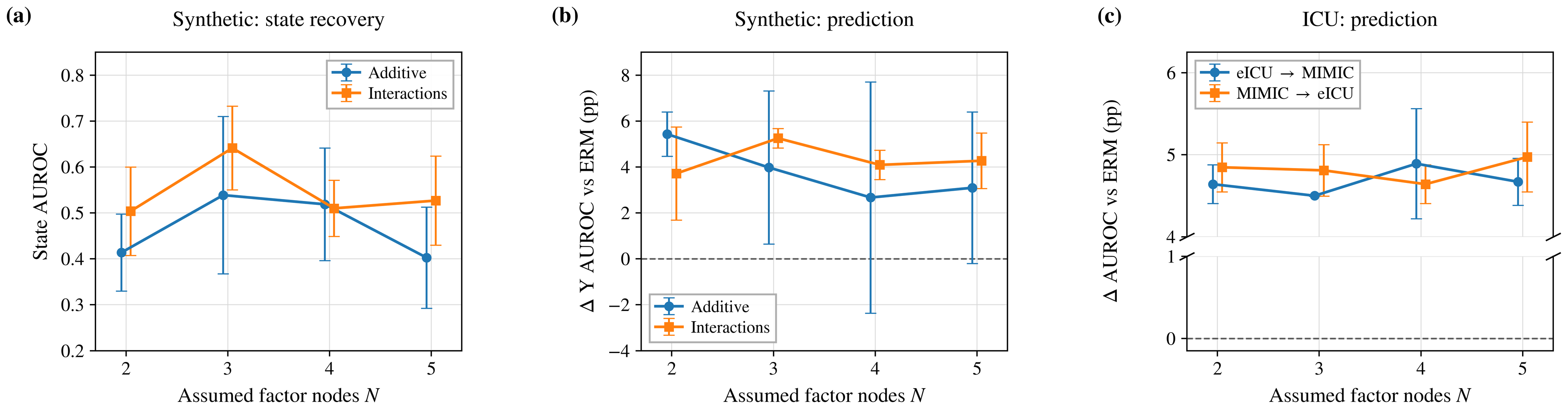}
\caption{\textbf{Number of assumed latent factors.} (a) State AUROC and (b) outcome AUROC gain over a transformer ERM, in percentage points, on a generator with three states and three latent factors, when \method{} assumes $N$ factor nodes. (c) AUROC gain over ERM in the ICU study, where the number of factors is unknown (broken axis). Points are means over three seeds.}
\label{fig:nfactors}
\end{figure*}

\subsection{Analysis on real-world datasets}
\label{sec:exp-real}
\subsubsection{Cross-database ICU mortality}
\label{sec:exp-icu}
Table~\ref{tab:icu} shows that \method{} reaches $0.818$ and $0.812$ AUROC, $0.036$ and $0.035$ above matched ERM, with the best AUPRC and NLL in both directions. NLL falls from $0.383$ and $0.372$ to $0.316$ and $0.328$, so the improvement is not confined to ranking. Into eICU, \method{} exceeds the strongest historical baseline by $0.020$ AUROC; into MIMIC-IV, it is $0.002$ below the SWAD-style recipe on a different backbone. These comparisons assess recorded-outcome prediction under combined selection and measurement shifts; they do not verify recovery of a clinical disease state or the fixed-acquisition premise of the identification theorem.

\begin{table*}[t]
\centering\footnotesize
\setlength{\tabcolsep}{3pt}
\caption{\textbf{Cross-database ICU mortality} (E$\to$M: eICU to MIMIC-IV; M$\to$E: reverse). AUROC and AUPRC average physical hospitals; NLL pools target patients; mean$\pm$SD over three seeds. $\dagger$: source-trained baselines with a different backbone; the lower block shares the \method{} backbone and inputs. Dashes denote unreported NLL values.}
\label{tab:icu}\label{tab:icuabl}
\begin{tabular*}{\textwidth}{@{\extracolsep{\fill}}lcccccc@{}}
\toprule
& \multicolumn{2}{c}{AUROC} & \multicolumn{2}{c}{AUPRC} & \multicolumn{2}{c}{NLL}\\
\cmidrule(lr){2-3}\cmidrule(lr){4-5}\cmidrule(lr){6-7}
Method & E$\to$M & M$\to$E & E$\to$M & M$\to$E & E$\to$M & M$\to$E\\
\midrule
ERM early fusion$^\dagger$ & $0.808{\scriptstyle\pm.002}$ & $0.783{\scriptstyle\pm.015}$ & $0.467{\scriptstyle\pm.012}$ & $0.394{\scriptstyle\pm.027}$ & $0.342{\scriptstyle\pm.009}$ & $0.373{\scriptstyle\pm.057}$\\
ERM late fusion$^\dagger$ & $0.793{\scriptstyle\pm.011}$ & $0.792{\scriptstyle\pm.001}$ & $0.473{\scriptstyle\pm.018}$ & $0.411{\scriptstyle\pm.004}$ & $0.351{\scriptstyle\pm.029}$ & $0.347{\scriptstyle\pm.042}$\\
V-REx$^\dagger$ & $0.804{\scriptstyle\pm.009}$ & $0.778{\scriptstyle\pm.015}$ & $0.461{\scriptstyle\pm.010}$ & $0.388{\scriptstyle\pm.032}$ & $0.340{\scriptstyle\pm.020}$ & $0.368{\scriptstyle\pm.032}$\\
GroupDRO$^\dagger$ & $0.810{\scriptstyle\pm.002}$ & $0.778{\scriptstyle\pm.015}$ & $0.468{\scriptstyle\pm.014}$ & $0.389{\scriptstyle\pm.030}$ & $0.332{\scriptstyle\pm.013}$ & $0.378{\scriptstyle\pm.037}$\\
CORAL$^\dagger$ & $0.805{\scriptstyle\pm.003}$ & $0.774{\scriptstyle\pm.017}$ & $0.462{\scriptstyle\pm.015}$ & $0.383{\scriptstyle\pm.030}$ & $0.332{\scriptstyle\pm.013}$ & $0.362{\scriptstyle\pm.034}$\\
SWAD-style$^\dagger$ & $\mathbf{0.820}{\scriptstyle\pm.001}$ & $0.790{\scriptstyle\pm.013}$ & $0.485{\scriptstyle\pm.008}$ & $0.405{\scriptstyle\pm.023}$ & $0.323{\scriptstyle\pm.008}$ & $0.349{\scriptstyle\pm.045}$\\
\midrule
Matched rich ERM & $0.782{\scriptstyle\pm.017}$ & $0.777{\scriptstyle\pm.011}$ & $0.462{\scriptstyle\pm.022}$ & $0.398{\scriptstyle\pm.012}$ & $0.383{\scriptstyle\pm.030}$ & $0.372{\scriptstyle\pm.024}$\\
Information only & $0.788{\scriptstyle\pm.018}$ & $0.783{\scriptstyle\pm.015}$ & $0.470{\scriptstyle\pm.020}$ & $0.402{\scriptstyle\pm.015}$ & -- & --\\
Uncorrected state & $0.788{\scriptstyle\pm.018}$ & $0.782{\scriptstyle\pm.016}$ & $0.470{\scriptstyle\pm.020}$ & $0.401{\scriptstyle\pm.016}$ & -- & --\\
No edges & $0.788{\scriptstyle\pm.018}$ & $0.782{\scriptstyle\pm.016}$ & $0.470{\scriptstyle\pm.020}$ & $0.401{\scriptstyle\pm.016}$ & -- & --\\
\textbf{\method{}} & $0.818{\scriptstyle\pm.018}$ & $\mathbf{0.812}{\scriptstyle\pm.016}$ & $\mathbf{0.502}{\scriptstyle\pm.020}$ & $\mathbf{0.431}{\scriptstyle\pm.016}$ & $\mathbf{0.316}{\scriptstyle\pm.016}$ & $\mathbf{0.328}{\scriptstyle\pm.028}$\\
\bottomrule
\end{tabular*}
\end{table*}

\subsubsection{Matched architecture ablations}
\label{sec:icuabl}
The five clinical variants share all prediction inputs, the physiology lattice, the context branch, the warm-up, and source-only model selection. Table~\ref{tab:icuabl} isolates the information objective, the correction, and the structural edges. Information learning alone raises AUROC from $0.782$ to $0.788$ into MIMIC-IV and from $0.777$ to $0.783$ into eICU. Adding uncorrected state evidence or removing the mechanism edges leaves performance at approximately this level. Conditional evidence raises AUROC to $0.818$ and $0.812$, with corresponding AUPRC gains. These matched comparisons support the contribution of the conditional mechanism branch within this architecture; the historical baselines in Table~\ref{tab:icu} use different backbones and training schedules.

\subsubsection{Sensitivity to the number of latent factors}
\label{sec:exp-icu-factors}
Figure~\ref{fig:nfactors}c varies the assumed number of latent factors from two to five in the clinical model, where the true number is unknown.
On ICU the gain lies between about $4.5$ and $5$ points at every $N$ in both directions, and its variation across $N$, under half a point, is of the order of the spread across seeds, so the clinical result does not hinge on the assumed number. These gains exceed the $3.6$ and $3.5$ points of Table~\ref{tab:icu} over matched ERM; the sweep is a sensitivity analysis and is not pooled with Table~\ref{tab:icu}.

\subsubsection{Exploratory ADNI results}
The earlier categorical variant does not establish a clinical improvement claim. On the shared 217-patient target, source-strict \method{} has macro-AUROC $0.711\pm0.019$, compared with $0.738\pm0.004$ for early-fusion ERM; target-strict \method{} reaches $0.709\pm0.021$. The deep common-style variant reaches $0.761\pm0.017$ but has higher NLL ($0.937\pm0.031$) than early ERM ($0.917\pm0.043$). Its estimator, target-context access and shared upstream encoder differ from the ICU protocol, so these results delimit the empirical evidence rather than extend the ICU gain. Appendix~\ref{app:adni} reports all 20 methods and the full protocol.

\section{Conclusion}
\label{sec:conclusion}

This paper shows that enrollment shortcuts in medical multimodal data can be separated from disease evidence in a principled way, without observing the disease state. From a generative account in which cohorts differ in selection and measurement but share disease mechanisms, we showed that unobserved states can be aligned across cohorts, that enrollment-induced natural-parameter changes identify a minimal selection statistic, and that conditioning on it yields disease evidence identifiable without state labels and invariant to enrollment shifts within the identified family. The same analysis specifies what an estimator must preserve, namely cross-modal dependence and information not shared by all modalities, and explains why state evidence and endpoint prediction are related but distinct goals. \method{} turns these results into one jointly trained model, and experiments confirm the predictions: under reversed enrollment, conditioning on the estimated statistic lifts state recovery from below chance to near the oracle, unlike a random direction, and on cross-database ICU mortality \method{} improves AUROC, AUPRC and NLL over matched ERM in both directions.

The analysis also shows where the next gains lie. A reconstruction-trained lattice already retains the state information but does not yet express it in the common Gaussian coordinates our identification results assume; representations that do would extend the guarantees end to end. Selection statistics beyond rank one, a clinical anchor for the learned states, and more seeds on unseen targets would further strengthen the clinical findings, which remain exploratory (Appendix~\ref{app:scope}). More broadly, separating what enrollment contributes from what the disease contributes offers a principled route to multimodal models whose disease evidence remains valid when enrollment changes.

% IEEE appendices remain in the same two-column article.

\bibliographystyle{IEEEtranN}
\bibliography{references,references_added}

\clearpage
\newpage
\appendices
\section{Full statements and proofs}
\label{app:proofs}

Unless stated otherwise, Equation~\eqref{eq:mixture} holds in common coordinates $z\in\mathbb R^d$, with positive component weights, distinct components within each cohort, and positive-definite state-specific covariances $\Sigma_s$ shared across cohorts. We use $\Gamma$ for a basis of the selection subspace, reserving $B$ for background variables. All identification statements concern population laws. No claim of finite-sample recovery or convergence of the joint training objective is needed for these results.

\subsection{Finite mixtures and Proposition~\ref{prop:align}}
\label{app:align}

\begin{fullstatement}{Proposition~\ref{prop:align}}
The components of each source mixture are identified up to a permutation. If the shared covariances are pairwise distinct, the components can be matched across all source cohorts up to one global permutation. For the three-state common-covariance family in Equation~\eqref{eq:parallel}, linear independence of $a$ and $\nu$ and three distinct source values of $\kappa_e$ identify the central state and match the end states up to their global exchange.
\end{fullstatement}

\begin{proof}
We recall the finite-Gaussian-mixture argument~\citep{yakowitz1968identifiability}. If two minimal mixtures agree, combining identical components gives $\sum_{j=1}^J\alpha_j\mathcal N(z;\mu_j,\Sigma_j)=0$ for distinct parameter pairs. Choose $v$ such that the pairs $(v^\top\Sigma_jv,v^\top\mu_j)$ are all distinct. Such a $v$ exists because each pairwise failure is contained in a proper polynomial zero set. Taking moment-generating functions along $v$ yields
\[
\sum_{j=1}^J\alpha_j
\exp\!\left(\tfrac12 t^2v^\top\Sigma_jv+t v^\top\mu_j\right)=0.
\]
Order the pairs first by projected variance and then by projected mean. Divide by the exponential for the largest pair and let $t\to+\infty$. All other terms vanish, so its coefficient is zero. Repeating proves linear independence and hence identification of the unordered component parameters and weights.

When the shared covariances are distinct, a given covariance has exactly one match in each cohort. Matching by covariance leaves only one global ordering of the states unspecified.

For Equation~\eqref{eq:parallel}, the arithmetic mean of the three unordered component means is $a$, regardless of their mixture weights. Linear independence makes the component at $a$ unique. After subtracting $a$, the other means lie on
\[
L_+=a+\operatorname{span}(\nu),\qquad
L_-=-a+\operatorname{span}(\nu).
\]
These lines are distinct and parallel, and each contains at least three distinct observed points. Any line through three points in their union must contain two points from one of them, hence coincide with it; it cannot then intersect the other parallel line. The two lines and their component assignments are therefore determined, up to their exchange. Neither the values of $\kappa_e$ nor a source cohort with $\kappa_e=0$ need be known.
\end{proof}

\subsection{Proof of Theorem~\ref{thm:evidence} and its estimator correspondence}
\label{app:evidence}

\begin{fullstatement}{Theorem~\ref{thm:evidence}}
Under Equation~\eqref{eq:mixture} and an alignment condition of Proposition~\ref{prop:align}, the subspace $\mathcal U$ in Equation~\eqref{eq:B} is identified in the common coordinates. A linear statistic $A^\top Z$ determines every same-state source density ratio if and only if $\mathcal U\subseteq\operatorname{col}(A)$. Let $r=\dim\mathcal U$, choose a basis $\Gamma\in\mathbb R^{d\times r}$, and choose $C$ so that $Z\leftrightarrow(U,V)=(\Gamma^\top Z,C^\top Z)$ is invertible. Then $p_e(v\mid u,s)=h_s(v\mid u)$ is independent of the source cohort. The evidence in Equation~\eqref{eq:target} equals $\log h_s(v\mid u)-\log h_0(v\mid u)$. The reference posterior is identified up to the global state permutation, and the evidence with the corresponding baseline. The same conditional kernels apply to any Gaussian target with covariance $\Sigma_s$ in state $s$ and $\Sigma_s^{-1}(\mu_{ts}-\mu_{e_0s})\in\mathcal U$ for every state.
\end{fullstatement}

\begin{proof}
Define $\beta_{es}=\Sigma_s^{-1}(\mu_{es}-\mu_{e_0s})$. Expanding the Gaussian quadratic forms gives
\begin{equation}
\begin{aligned}
\frac{p_e(z\mid s)}{p_{e_0}(z\mid s)}
&=\exp\!\left(\beta_{es}^\top z-c_{es}\right),\\
c_{es}&=\tfrac12\left(
\mu_{es}^\top\Sigma_s^{-1}\mu_{es}
-\mu_{e_0s}^\top\Sigma_s^{-1}\mu_{e_0s}\right).
\end{aligned}
\label{eq:ratio}
\end{equation}
Aligned mixture parameters determine these vectors and their span. Changing the reference cohort does not change the span, because every difference between two cohort means is a difference of their contrasts with $e_0$.

Since $\beta_{es}\in\mathcal U$, the right-hand side of Equation~\eqref{eq:ratio} is a positive function $a_{es}(u)$. The common change-of-coordinates Jacobian cancels, so
\[
\begin{aligned}
p_e(u,v\mid s)&=a_{es}(u)p_{e_0}(u,v\mid s),\\
p_e(u\mid s)&=a_{es}(u)p_{e_0}(u\mid s).
\end{aligned}
\]
The second identity follows by integration over $v$. Dividing gives $p_e(v\mid u,s)=p_{e_0}(v\mid u,s)=h_s(v\mid u)$. Factoring the densities in $(u,v)$ then proves the expression for $T$. Equivalently, the conditional law of $Z$ is the same on each affine conditioning fibre.

For minimality, suppose a ratio is measurable with respect to $A^\top Z$ but its vector $\beta_{es}$ is outside $\operatorname{col}(A)$. After removing redundant columns of $A$, a nondegenerate Gaussian conditioned on $A^\top Z$ has full support on its affine fibre. There is $w\in\ker(A^\top)$ with $\beta_{es}^\top w\neq0$, so $\beta_{es}^\top Z$ has positive conditional variance and cannot be measurable with respect to $A^\top Z$. This contradiction proves necessity. Conversely, $\mathcal U\subseteq\operatorname{col}(A)$ makes every expression in Equation~\eqref{eq:ratio} a function of $A^\top Z$, proving sufficiency. This is minimality for the specified ratio family, not identification of a uniquely defined biological nuisance.

Replacing $\Gamma$ by $\Gamma H$, with invertible $H$, replaces $U$ by $H^\top U$. Its state density ratios are unchanged because their common Jacobian cancels. A state permutation $\tau$ gives
\[
\begin{aligned}
T'_s&=T_{\tau(s)}-T_{\tau(0)},\qquad
\pi'_s=\pi_{\tau(s)},\\
q_{\pi'}^{\mathrm{ref}\prime}(s\mid z)
&=\qref(\tau(s)\mid z).
\end{aligned}
\]
Thus the posterior is permuted, while the evidence is also recentered at its newly chosen baseline.

For a target satisfying (A4), Equation~\eqref{eq:ratio} holds with $e=t$ and $\beta_{ts}\in\mathcal U$, giving the same cancellation. This is an assumption on the target family, not a property established for arbitrary targets from source data. When $r=0$, $U$ is constant and no correction is required. When $r=d$, $V$ is empty, $T=0$, and $\qref=\pi$.
\end{proof}

\paragraph*{Centred reference family} The centred-span identity and the Gaussian expression in Equations~\eqref{eq:centered_span} and~\eqref{eq:gaussian_evidence} of Section~\ref{sec:method} implement the same conditional kernels. Each centred contrast is a linear combination of reference-cohort contrasts, and subtracting the centred contrast for the reference cohort proves the converse inclusion.

\paragraph*{Conditional sub-model evidence}
If background coordinates $B$ are available, remove linear redundancies from $(B,U)$ to obtain $W$, and choose $V_D$ so that $Z\leftrightarrow(W,V_D)$ is invertible. Since $W$ determines $U$, the same cancellation gives the cohort-invariant ratio
\begin{equation}
\begin{aligned}
T_s^{V_D\mid W}
&=\log\frac{p_e(V_D\mid W,s)}{p_e(V_D\mid W,0)}\\
&=\log\frac{p_e(V_D,W\mid s)}{p_e(V_D,W\mid0)}
-\log\frac{p_e(W\mid s)}{p_e(W\mid0)}.
\end{aligned}
\label{eq:subscm}
\end{equation}
Conditioning on the larger $W$ can remove additional evidence; this quantity need not equal the original $T_s$ unless the added coordinates contribute no further conditional state evidence. For a known acyclic graph, state-independent factors cancel from the joint ratio, but the marginal normalizer must still be retained. This identity does not identify the graph. In $D=aS+\gamma B+\epsilon$, for example, the conditional mechanism is expressed through $D-\gamma B$, so interpreting $D$ requires retaining $B$ as a covariate.

\subsection{Connection to the structural model and the rank-one setting}
\label{app:structural}

Let $M_s=I-A_s$ and consider the linear Gaussian instance
\[
Z=M_s^{-1}(o_s+\kappa_e\ell_s r_0+\epsilon_s),\qquad
\epsilon_s\sim\mathcal N(0,\Omega_s),
\]
where $M_s$ is invertible, $\Omega_s$ is positive diagonal, $r_0$ is a fixed enrollment direction, and only $\kappa_e$ varies across cohorts. Here $r_0$ denotes a vector, distinct from the rank $r=\dim\mathcal U$. Then $\mu_{es}=M_s^{-1}(o_s+\kappa_e\ell_s r_0)$ and $\Sigma_s=M_s^{-1}\Omega_sM_s^{-\top}$, giving
\begin{equation}
\Sigma_s^{-1}(\mu_{es}-\mu_{e_0s})
=(\kappa_e-\kappa_{e_0})\ell_s M_s^\top\Omega_s^{-1}r_0.
\label{eq:scm_natural_shift}
\end{equation}
Provided the source values of $\kappa_e$ vary, the exact subspace is
\begin{equation}
\mathcal U=\operatorname{span}\!\left\{
\ell_s M_s^\top\Omega_s^{-1}r_0:s=0,\ldots,K-1\right\}.
\label{eq:scm_selection_span}
\end{equation}
A single enrollment parameter therefore yields rank one only when its nonzero state-specific natural-parameter directions are collinear. State-dependent propagation and innovation variances can otherwise produce several directions.

One sufficient structural condition is that $r_0$ is supported on root background coordinates. Their rows in $A_s$ are zero; because $\Omega_s$ is diagonal, $A_s^\top\Omega_s^{-1}r_0=0$, so the direction reduces to $\Omega_s^{-1}r_0$. If the innovation variances are equal over the selected roots within each state, this vector is proportional to $r_0$. With at least one nonzero $\ell_s$ and a nonzero source contrast, $\mathcal U=\operatorname{span}(r_0)$. The variances may differ between states. Selection on a single root coordinate is a special case.

More generally, suppose enrollment changes root-background means, all such natural contrasts remain in the background coordinate subspace, and the source contrasts span that entire subspace. Then $U$ is an invertible linear function of $B$ and
\[
T_s(z)=\log\frac{p(d\mid b,s)}{p(d\mid b,0)}.
\]
A sufficient excitation condition is that, in at least one state, the root-background mean contrasts span all background coordinates; the relevant diagonal covariance scaling is invertible. Without full coverage, $U$ need not recover all of $B$, and the interpretation as the entire conditional physiological mechanism is not justified. The theorem still applies to its identified subspace and targets satisfying (A4).

\subsection{Common transformations and nonlinear limits}
\label{app:transform}

For a common affine change $Z'=LZ+\tau$ with invertible $L$,
\[
\begin{aligned}
\mu'_{es}&=L\mu_{es}+\tau,&\Sigma'_s&=L\Sigma_sL^\top,\\
\beta'_{es}&=L^{-\top}\beta_{es},&\mathcal U'&=L^{-\top}\mathcal U.
\end{aligned}
\]
Choosing $\Gamma'=L^{-\top}\Gamma$ gives $U'=U+\Gamma^\top L^{-1}\tau$. Common Jacobians and the common translation of $U$ cancel from the state density ratios. Distinct-covariance matching and the centred parallel-line geometry are preserved; the zero-offset parametrization in Equation~\eqref{eq:parallel} need not remain literal after translation. Truncated Euclidean singular vectors and a fixed numerical ridge need not have this exact affine equivariance.

For a common diffeomorphism $F$, the transformed density is
\begin{equation}
p^F_e(F(z)\mid s)
=\frac{p_e(z\mid s)}{|\det J_F(z)|},\qquad
U^F=U\circ F^{-1}.
\label{eq:transported_evidence}
\end{equation}
The Jacobian cancels from every state likelihood ratio, while $U^F(F(Z))=U(Z)$ has exactly the same conditional distribution. Both terms of Equation~\eqref{eq:target}, hence $T$ and the reference posterior, are unchanged. A nonlinear $F$ need not preserve Gaussianity or make $U^F$ linear. Evidence preservation is therefore weaker than the coordinate condition required by the Gaussian estimator.

Likewise, an invertible calibration cannot recreate information already lost by an encoder: a measurable bijection with measurable inverse preserves the generated sigma-algebra and unrestricted Bayes risks. It may improve a restricted readout by changing the coordinates in which that readout operates. Gaussian-mixture identifiability alone does not establish identifiability of an unknown nonlinear observation map. Without restrictions on the component family, any density $p(z)$ admits latent explanations
\[
\pi_s=\int p(z)w_s(z)\,dz,\qquad
p(z\mid s)=\frac{p(z)w_s(z)}{\pi_s},
\]
for arbitrary positive functions $w_s$ summing to one. These explanations have the same observed law and different state posteriors. Additional views that only invertibly re-express the same $Z$ do not remove this ambiguity.

\subsection{Reference posterior and the exact target discrepancy}
\label{app:gap}

Let $\pi$ be a positive reference prior. For any probability law $\nu_U$ of $U$, define the reference population
\begin{equation}
P^{\mathrm{ref}}_{\nu_U}(S=s,du,dv)
=\pi_s\nu_U(du)h_s(v\mid u)\,dv.
\label{eq:reference_population}
\end{equation}
Here $S$ and $U$ are independent before $V$ is drawn. Bayes' rule gives
\[
\qref(s\mid u,v)
=\frac{\pi_sh_s(v\mid u)}{\sum_j\pi_jh_j(v\mid u)},
\]
independently of the chosen $\nu_U$. This defines the reference posterior even when no observed cohort has this sampling law.

For a target $e$ with these conditional kernels, write
\[
\begin{aligned}
g_s(u)&=P_e(S=s\mid U=u),\\
m_g(v\mid u)&=\sum_sg_s(u)h_s(v\mid u),\\
m_\pi(v\mid u)&=\sum_s\pi_sh_s(v\mid u).
\end{aligned}
\]
Define $R^S_e(q)=\mathbb E_e[-\log q(S\mid Z)]$ and assume the displayed expectations are finite. At fixed $u$, the KL divergence between the joint laws $g_s(u)h_s(v\mid u)$ and $\pi_sh_s(v\mid u)$ is $\mathrm{KL}(g(u)\|\pi)$. Factoring instead in the order $V$ then $S$ gives
\begin{equation}
\begin{aligned}
\Delta^S_e
&:=R^S_e(\qref)-R^S_e(p_e)\\
&=\mathbb E_{U,V}\,\mathrm{KL}\!\left(
p_e(\cdot\mid U,V)\,\|\,\qref(\cdot\mid U,V)\right)\\
&=\mathbb E_U\!\left[
\mathrm{KL}(g(U)\|\pi)
-\mathrm{KL}(m_g(\cdot\mid U)\|m_\pi(\cdot\mid U))
\right].
\end{aligned}
\label{eq:gap_exact}
\end{equation}
Nonnegativity of the second KL term proves Equation~\eqref{eq:gap}. Also,
\begin{equation}
\mathbb E_U\mathrm{KL}(g(U)\|\pi)
=I_e(S;U)+\mathrm{KL}(P_e(S)\|\pi).
\label{eq:gap_prior_decomposition}
\end{equation}
This separates state--selection dependence from prevalence mismatch. If $g(U)=\pi$ almost surely, both the bound and the actual excess risk vanish. Otherwise the reference posterior can differ from the target posterior. The actual gap is zero precisely when the two posteriors agree almost surely; in the positive Gaussian family this requires $g(U)=\pi$ almost surely. In particular, switching off an enrollment parameter is not by itself enough: it must also give state-independent $U$ and the chosen reference prevalence. Neither the bound nor its decomposition proves monotonicity of the actual excess risk along an arbitrary enrollment path.

The construction extends to any positive reference gate $a_s(u)$ by replacing $\pi_s$ with $a_s(u)$ in the reference posterior and the KL expressions. Comparing two such upper bounds alone does not order their actual risks.

\subsection{Proof of Proposition~\ref{prop:decomp} and the continuous-information limit}
\label{app:decomp}

\paragraph*{Definition on observed inputs}
For the exact link to Equation~\eqref{eq:gap}, require a measurable recovery map $Z=\Phi(X)$ in the fixed acquisition context and the observation-model condition $S\perp X\mid Z$. The latter follows from the factorization in Section~\ref{sec:setting}, where the measurement law depends on $Z$ but not additionally on $S$. Then
\[
\begin{aligned}
p^X_e(X)&=p_e(\cdot\mid Z),\\
q^{X,\mathrm{ref}}_\pi(X)&=\qref(\cdot\mid Z)
\quad\text{almost surely}.
\end{aligned}
\]
Without recoverability, the pointwise identification of these objects is not available; a posterior on raw inputs requires marginalizing over the unobserved coordinates.

\begin{proof}[Proof of Proposition~\ref{prop:decomp}]
Suppress the argument $x$. Insert $h(\hat q,\hat r)$ and $h(p^X_e,r_e)$ between $\hat f$ and $\eta_e$. By the definition of $\epsilon_h$, the Lipschitz property, and the residual bound,
\[
|\hat f-\eta_e|
\leq\epsilon_h+L_S\|\hat q-p^X_e\|_1
+L_R\|\hat r-r_e\|+d_e.
\]
Inserting $q^{X,\mathrm{ref}}_\pi$ into the state difference and applying the triangle inequality gives $A_e$. For $R^Y_e(f)=\mathbb E_e[(f(X)-Y)^2]$, conditioning on $X$ yields
\[
R^Y_e(\hat f)-R^Y_e(\eta_e)
=\mathbb E_e[(\hat f(X)-\eta_e(X))^2]
\leq\mathbb E_e[A_e(X)^2].
\]
This uses the scalar binary Brier convention; summing squared errors over both class probabilities multiplies both sides by two.
\end{proof}

Under the recoverability and measurement conditions above, Pinsker's inequality and Equation~\eqref{eq:gap_exact} give the precise connection between state log-loss and the third term:
\begin{equation}
\begin{aligned}
\mathbb E_e\|p^X_e(X)-q^{X,\mathrm{ref}}_\pi(X)\|_1^2
&\leq 2\,\Delta^S_e\\
&\leq 2\,\mathbb E_U\mathrm{KL}(g(U)\|\pi).
\end{aligned}
\label{eq:observable_gap}
\end{equation}
The logarithms are natural. The log-loss gap and the posterior norm are distinct quantities; Equation~\eqref{eq:observable_gap} connects them under the stated conditions.

\paragraph*{Correspondence with the jointly trained head}
Equation~\eqref{eq:joint_outcome} uses log odds rather than probabilities as numerical inputs. Its composite probability-space head is
\[
\begin{aligned}
\hat h(q,r)&=h_\theta\!\left(r,
\left\{\log\frac{q_s}{q_0}\right\}_{s=1}^{K-1}\right),\\
\log\frac{\hat q_\pi(s\mid\hat z)}{\hat q_\pi(0\mid\hat z)}
&=\log\frac{\pi_s}{\pi_0}+\hat T_s(\hat z).
\end{aligned}
\]
The positive Gaussian densities and positive prior ensure these odds are defined. This substitution puts the method in the form of Proposition~\ref{prop:decomp}. Lipschitz continuity is assumed for the comparison map $h$; it is not inferred merely from implementing a neural head on log odds. The term $\epsilon_h$ measures approximation of that comparison map at the learned inputs. Joint updates can change all estimated quantities, so the decomposition does not assert that training affects only one term. The residual $d_e$ includes label-protocol drift and any failure of $(p^X_e,r_e)$ to describe the endpoint through a common $h$.

\paragraph*{What a state-only outcome head can lose}
Let $Q=q^{X,\mathrm{ref}}_\pi(X)$ be the only input to an unrestricted predictor. Since $Q$ is measurable with respect to $X$,
\begin{equation}
\inf_f R^Y_e(f(Q))-R^Y_e(\eta_e)
=\mathbb E_e\operatorname{Var}(\eta_e(X)\mid Q).
\label{eq:state_only_limit}
\end{equation}
The minimizer is $f(Q)=\mathbb E_e[\eta_e(X)\mid Q]$. Adding a deterministic representation $R=r(X)$ reduces this oracle risk by
\[
\mathbb E_e\!\left[
\left(\mathbb E_e[\eta_e(X)\mid Q,R]
-\mathbb E_e[\eta_e(X)\mid Q]\right)^2\right]\geq0,
\]
by orthogonality of nested conditional expectations. This is a within-target information statement; source-trained heads need not attain either oracle. It explains why retaining continuous physiology can matter even with exact reference evidence.

\subsection{Interaction example and averaging}
\label{app:interaction}

Section~\ref{sec:theory} uses the following example to show that reducing each modality to its own state posterior before fusion can discard all state evidence.

\begin{proposition}[Marginal-silent interaction]\label{prop:interaction}
Let two zero-mean, unit-variance Gaussian views have correlation $+\rho$ in one state and $-\rho$ in the other, where $0<\rho<1$, with equal state priors. Both single-view posteriors are constant, so every classifier based only on those posteriors has error $1/2$. The joint log-likelihood ratio, with the positive-correlation state in the numerator, is $2\rho x_1x_2/(1-\rho^2)$, and the joint Bayes error is $\arccos(\rho)/\pi$.
\end{proposition}

\begin{proof}[Proof of Proposition~\ref{prop:interaction}]
Let the two zero-mean state covariances be
\[
\Sigma_+=\begin{pmatrix}1&\rho\\\rho&1\end{pmatrix},\qquad
\Sigma_-=\begin{pmatrix}1&-\rho\\-\rho&1\end{pmatrix},\qquad 0<\rho<1.
\]
Both single-view laws are $\mathcal N(0,1)$ in both states. With equal priors, each single-view posterior is therefore the constant vector $(1/2,1/2)$, and a classifier receiving only these vectors has error $1/2$.

The determinants coincide and expanding the inverse-covariance quadratic forms gives
\[
\log\frac{p_+(x_1,x_2)}{p_-(x_1,x_2)}
=\frac{2\rho x_1x_2}{1-\rho^2}.
\]
The joint Bayes classifier chooses the positive-correlation state when $x_1x_2>0$. Under that state, set $A=(X_1+X_2)/\sqrt2$ and $B_*=(X_1-X_2)/\sqrt2$. These variables are independent Gaussians with variances $1+\rho$ and $1-\rho$. For independent standard normals $G_1,G_2$, the error probability is
\[
\begin{aligned}
P_+(X_1X_2<0)
&=P\!\left(\left|\frac{G_1}{G_2}\right|
<\sqrt{\frac{1-\rho}{1+\rho}}\right)\\
&=\frac{2}{\pi}\arctan\sqrt{\frac{1-\rho}{1+\rho}}
=\frac{\arccos\rho}{\pi}.
\end{aligned}
\]
Here $G_1/G_2$ is standard Cauchy. Reflection of one view gives the same error under the other state, proving the stated Bayes error. The limitation concerns fusion of the marginal posterior outputs, not every classifier that receives the original views.
\end{proof}

For completeness, averaging probabilistic predictions obeys a different, purely convexity-based statement. For weights $\alpha_m\geq0$ summing to one, the log loss of $\bar q=\sum_m\alpha_m q_m$ satisfies
\[
-\log\bar q(Y\mid X)\leq\sum_m\alpha_m[-\log q_m(Y\mid X)],
\]
and squared loss satisfies the analogous inequality. These bounds compare an average with the weighted average of its constituent losses; they do not guarantee improvement over the best constituent and cannot recover information absent from every constituent output.

\subsection{Evidence-preserving representations: proof of Proposition~\ref{prop:lattice}}
\label{app:lattice}

\paragraph*{Shared content and exact-support blocks}
Include private variation in $\tilde Z=(Z,\xi_{1:M})$ and let $\tilde A_m$ index the coordinates observed by modality $m$. For a nonempty modality subset $a$, define
\[
\begin{aligned}
I(a)&=\bigcap_{m\in a}\tilde A_m,\qquad
C(a)=\bigcup_{a'\supsetneq a}I(a'),\\
J(a)&=I(a)\setminus C(a).
\end{aligned}
\]
The union is empty when $a=[M]$. The set $I(a)$ is all content shared by the views in $a$; $J(a)$ contains only coordinates observed by exactly those views. These objects must not be conflated.

For the positive shared-content result we invoke \citet[Thm.~3.2]{yao2024multiview} with its full setup: smooth invertible view maps, a smooth positive latent density on the stipulated product domain, the correct shared-content dimension, smooth encoders into a unit cube of that dimension, and a global optimum of the population alignment-minus-entropy objective. The conclusion identifies $\tilde Z_{I(a)}$ up to a smooth invertible map. The result does not by itself separate $J(a)$ from $C(a)$, or prove identification for a finite-sample surrogate.

\begin{proof}[Proof of the remaining claims in Proposition~\ref{prop:lattice}]
To see the conditional ambiguity, choose a smooth change
\[
(\tilde z_{J(a)},\tilde z_{C(a)})
\longmapsto
(\psi(\tilde z_{J(a)};\tilde z_{C(a)}),\tilde z_{C(a)}),
\]
whose full extension $\Psi$ is a diffeomorphism on the latent domain; all other coordinates are unchanged. Such transformations include blockwise diffeomorphisms and admissible conditional triangular changes. If $m\notin a$, it observes no coordinate in $J(a)$ and its inputs are unchanged. If $m\in a$, it observes all coordinates in $J(a)\cup C(a)$ and its view map can be composed with the corresponding restriction of $\Psi^{-1}$. The transformed model has exactly the same observed views. Thus view geometry alone cannot rule out this conditional reparameterization. This is an ambiguity of the unrestricted smooth view model; additional distributional restrictions can reduce it.

For evidence invariance, Equation~\eqref{eq:transported_evidence} shows that the state-density Jacobian cancels and that the transported statistic is the same random variable. Hence both likelihood-ratio terms in Equation~\eqref{eq:target} are unchanged, proving part (ii)'s transformation claim.

For reconstruction, the inverse view maps and coverage of all coordinates give a measurable map $Z=\Phi(X)$ at the fixed acquisition context. Let $R=E(X)$ and suppose a decoder $G$ satisfies $G(E(X))=X$ almost surely on the population under consideration. Then
\[
\begin{aligned}
Z&=\Phi(G(R)),\\
T(Z)&=(T\circ\Phi\circ G)(R),\\
\qref(\cdot\mid Z)&=\qref(\cdot\mid\Phi(G(R))).
\end{aligned}
\]
The evidence and posterior are therefore measurable functions of $R$. For a claim on a target population, the reconstruction identity must hold there as well; reconstruction of finitely many training examples is insufficient.

Finally, let $S\in\{0,1\}$, let $Z_{\mathrm{sh}}\sim\mathcal N(0,1)$ be independent of $S$, and let $Z_{12}=S+\epsilon$ with independent $\epsilon\sim\mathcal N(0,1)$. Take three views with identity maps: the first two observe $(Z_{\mathrm{sh}},Z_{12})$ and the third observes only $Z_{\mathrm{sh}}$. All cohorts have the same state-conditional laws, so $U$ is trivial and $T_1(Z)=Z_{12}-1/2$. This nonconstant evidence cannot be recovered from the content $Z_{\mathrm{sh}}$ shared by all three views. The example isolates preservation of evidence; it does not require solving a cross-cohort component alignment problem.
\end{proof}

The information objective in Equation~\eqref{eq:lattice} is motivated by these distinctions. Alignment encourages common content, reconstruction encourages retaining the rest, and residual terms organize partially shared information after broadly shared blocks. Ordinary subtraction in observation space does not establish a unique conditional block decomposition under arbitrary nonlinear view maps. Exact reconstruction is sufficient, but not necessary, to preserve a particular $T$; a lossy statistic could retain that evidence alone. Conversely, an injective encoder may preserve all information while presenting it in coordinates that violate Equation~\eqref{eq:mixture}. These limits distinguish information preservation, coordinate identification, and the behavior of the joint estimator in Equation~\eqref{eq:discern_joint_objective}.

\section{Synthetic benchmark: protocol and supplementary results}
\label{app:synthetic}

\subsection{Evidence levels and comparability}
The common-coordinate study supplies the generator's ten-dimensional $Z$ to all compared state estimators. This removes representation error but does not reveal $S$ during fitting; it is a mechanism-stage experiment, not an end-to-end raw-input result. The learned-view study uses the original multimodal observations and an estimated 13-dimensional lattice. The clinical study is a separate, label-informed architecture with 38 physiology dimensions. These three studies must not be pooled into one effect estimate or presented as runs of an identical estimator.

The five common-coordinate seeds are 58013, 58027, 58043, 58061 and 58073; the learned-view study uses the first three in both regimes, retaining the paired generated studies; the attribution study uses five separately generated studies with seeds 31415, 27182, 16180, 14142 and 17320. A seed in simulation is one generated study with multiple source and target observations. Tables report sample standard deviations across seeds, not standard errors, and small numbers of seeds are not interpreted as significance tests. All four target laws are reported: stronger selection is informative because a method can benefit from selection under one law and fail under its reversal. Equality of an evidence function across cohorts does not imply exactly equal marginal AUROC, because the distribution of conditioning covariates and the label mechanism can change task difficulty; we therefore do not impose a flat-AUROC validity criterion.

\subsection{Generator and source fitting}
The generator has three latent states, ten factors and three leaky-affine observation views with private variation. Source enrollment acts on the generator's selection-related root factors, which play the role of $B$ in Figure~\ref{fig:motivation}a, and is strong and heterogeneous across the twelve cohorts; label styles, the measurement of $Y$, are heterogeneous. Target-law generation changes enrollment as stated in Section~\ref{sec:experiments} together with site and label style under the original law definitions. The recorded binary $Y$ uses a score involving continuous latent factors (including $z_0$ and $0.7z_2$), cohort-dependent terms that do not involve $S$, private variation and site effects, which together play the role of $c_Y(e)$ and the label noise in Equation~\eqref{eq:cohort}; it is not replaced with a label generated only from $S$, because a clean state task and an observed outcome task test different claims.

The state mixture uses $K=3$, at most 150 EM iterations, two restarts and covariance ridge $10^{-3}$. Source-validation density selects common versus state-specific covariance. Common-covariance states are aligned by the parallel-line estimator of Equation~\eqref{eq:parallel}, with a check that the permutations preserve the observation density; distinct state covariances provide the alternative alignment of Proposition~\ref{prop:align}. Finite-sample EM and geometric fitting approximate the population construction, and the generator is not claimed to satisfy Equation~\eqref{eq:parallel} exactly. The natural-shift approximation keeps one direction. For averaging, the original fit and two within-cohort bootstrap fits are matched through source-validation posterior agreement and combined with equal weights. True $S$ is excluded from all these choices; the single global source-$S$ scoring permutation is determined after state models and outcome heads are fitted.

\subsection{Outcome heads on common coordinates}
Outcome heads have twelve inputs (ten coordinates and two state-evidence slots), two hidden layers of width 64 and a binary output. The ERM control fills the evidence slots with zero; state-only ablations zero the continuous coordinates. AdamW uses learning rate $10^{-3}$, weight decay $10^{-4}$, batch size 512, gradient clipping at 5, at most 40 epochs and patience 6 on source cohort-average NLL. The style variants add a centred cohort-specific linear logit and intercept, initialized at zero; deployment always uses the shared head. No weight averaging is used in these heads. Table~\ref{tab:outcome} gives the proper scores under reversal. For state recovery under reversal, single corrected-fit state NLL is $0.787$ and $0.891$ (additive, with interactions), and averaged corrected-fit NLL is $0.781$ and $0.883$; these are modest improvements, distinct from the large correction-versus-uncorrected contrast of Table~\ref{tab:state}. A joint criterion, requiring both that state averaging lowers state NLL relative to the single corrected fit and that the averaged-evidence-plus-style head lowers outcome NLL relative to ERM, holds in 1/5 additive and 2/5 interaction studies.

\subsection{Attribution study}
Figure~\ref{fig:attrib} tests whether the correction comes from the estimated selection direction rather than from conditioning as such. Five separately generated studies use the same twelve-source sample budgets, $K=3$, rank-one correction, covariance-family selection and four target laws. Within each study a single fitted mixture supplies all three readouts of Figure~\ref{fig:attrib}; the random unit direction is fixed by the study seed plus 99013. Source $S$ supplies only the final scoring permutation. Reference KL is $\mathbb E_Z\,\mathrm{KL}(q^{\mathrm{ref}}(\cdot\mid Z)\|\hat q(\cdot\mid Z))$ averaged over target observations, not actual-target excess log-loss. The aggregate file retains all 240 planned method--regime--law--seed rows, including oracle and linear-outcome diagnostics. These data are separate from the five averaged-fit studies and are not pooled with them. Their absolute levels are therefore not comparable with Table~\ref{tab:state}: with interactions, the estimated direction reaches $0.861$ here against $0.825\pm0.067$ for the averaged fit there, a difference within one between-study SD that reflects different generated studies rather than an effect of averaging. Within the studies of Table~\ref{tab:state}, averaging does not lower performance: reversed-law state NLL is $0.891$ for the single corrected fit and $0.883$ for the average (Appendix~\ref{app:synthetic}, outcome heads). Per study, the estimated direction exceeds both other readouts in all five studies of each regime; across studies, its state AUROC is $0.860\pm0.010$ and $0.861\pm0.033$, against $0.151\pm0.242$ and $0.270\pm0.178$ for the random direction and $0.115\pm0.188$ and $0.252\pm0.174$ uncorrected. The attribution study is designed for the within-mixture comparison of the three readouts, not for absolute levels.

\subsection{Learned representations, calibration and diagnostic probes}
\label{app:learned}
The lattice has atom dimensions $(2,1,1,1,3,2,3)$ with hidden width 32. Per-modality standardization uses source training statistics; site FiLM context is zero in both source and target, so the lattice receives no information about the target's site change, which Proposition~\ref{prop:lattice} holds fixed. The information objective is Equation~\eqref{eq:lattice}: within-cohort contrastive alignment, reconstruction, conditional-residual reconstruction and $0.1$ times a variance-floor penalty; AdamW with learning rate $10^{-3}$, weight decay $10^{-4}$, batch size 256, at most 30 epochs and patience 5 on source-validation information loss. No $S$ or $Y$ supervises the encoder.

For the calibration study we freeze the learned features $\tilde z$ and fit a shared, blockwise invertible map $F$ jointly with the mixture using the Jacobian-normalized feature likelihood
\begin{equation}
\begin{aligned}
\log p_e(\tilde z)
&=\log\sum_s\pi_{es}\,\mathcal N\big(F(\tilde z);\mu_{es},\Sigma_s\big)\\
&\quad+\log|\det J_F(\tilde z)|,
\end{aligned}
\label{eq:calib}
\end{equation}
so that $z=F(\tilde z)$; the Jacobian normalizes the feature density but cancels from posterior ratios. Calibration first standardizes features with the base mixture's source mean and scale. Three block-diagonal LU layers have positive diagonal factors and positive-slope piecewise-affine activations; blocks never mix different information atoms, and shifted kinks are initialized with unit slopes for an exact identity initialization. The affine control fixes activation slopes to one; the mixture-only control fixes the entire transform. All refinements start from the same fitted density; Adam with learning rate $10^{-3}$, batch size 512, clipping at 5, at most 40 epochs and patience 6 on source-validation feature NLL, with the initial model included as a candidate. The outcome head recipe matches the common-coordinate heads; all lattice-based heads have the same input width, and a calibrated-coordinate ERM control measures a coordinate change without added state evidence. Raw-concatenated-input ERM with the same hidden widths reaches reversed-selection outcome AUROC $0.792$ and $0.812$. Table~\ref{tab:learned} gives the learned-view results.

Nonlinear calibration improves source-validation feature-density NLL from $5.240$ to $4.233$ (additive) and from $5.925$ to $4.939$ (with interactions), while state NLL under reversal ($12.72\pm4.83$ base versus $12.89\pm4.93$ calibrated; $7.62\pm2.49$ versus $7.79\pm1.90$) does not improve; no seed has both primary NLL endpoints improved. Fitting the feature density more accurately is insufficient to identify the intended state family. No \method{} variant improves outcome AUROC over Lattice + ERM; affine calibration gives the lowest additive outcome NLL ($0.495$ against $0.507$), within one SD.

Diagnostic probes are fitted only after deployed models, checkpoints and predictions are saved. Linear multinomial probes select $C\in\{0.1,1,10\}$ by their own validation NLL; neural probes have two width-64 hidden layers, at most 30 epochs and patience 6. Source-trained probes use source $S$ and evaluate all target laws; on the frozen lattice they obtain source-validation AUROC $0.934$ and $0.889$ but reversed-target AUROC $0.173$ and $0.652$, so source-trained probes can themselves fail under reversal. Target-trained probes use 5{,}000 independent observations from a reserved target-adaptation stream (4{,}000 training, 1{,}000 validation) and evaluate the existing 10{,}000 target-test samples; they obtain $0.868$ and $0.832$ on the selection-off law and $0.993$ and $0.958$ on the reversed law, where the high values can exploit target-specific selection. Oracle-$Z$ selection-off values are $0.903$ and $0.859$. These results support residual state readability, not information-theoretic sufficiency or unsupervised disease identification; true target labels never initialize, select or update a deployed method.

Numerical checks verified inverse round trips, analytic versus automatic-differentiation Jacobians, atom block structure, identity initialization and one-dimensional density normalization. The main run completed six cases, yielding 240 state-metric rows, 240 outcome rows and 432 diagnostic-probe rows; thirty previous-baseline replay comparisons passed; 29 frozen source and protocol file hashes and 162 saved checkpoint hashes were verified, and diagnostic probes left deployed checkpoints unchanged.

\section{ICU experiment: protocol and supplementary results}
\label{app:icu}

\subsection{Cohort accounting}
Both directions preserve a fixed patient split and feature contract: 48-hour observations, a 6-hour gap and same-admission hospital mortality. The eICU-to-MIMIC target has 9{,}922 stays from 9{,}922 patients at one physical hospital; its 23 care-unit and time groups are not 23 hospitals. The MIMIC-to-eICU target has 26{,}763 stays from 22{,}982 patients at 57 physical hospitals. The saved configurations identify MIMIC-IV~2.0 and eICU-CRD~2.0. Eligibility requires age at least 18 and ICU length of stay from 54 to 240 hours, with a half-hour rounding tolerance. The retained source caches contain 18{,}419 MIMIC stays and 58{,}391 eICU stays; these are not the target-test counts. eICU preprocessing removes 11 stays with unknown sex and 869 stays to resolve cross-hospital patient duplicates; where cross-hospital time ordering is unavailable, the patient-to-hospital assignment uses the lexicographically smallest stay key. Training and validation use fixed patient-disjoint role plans (Table~\ref{tab:roles}); target adaptation partitions are held separately and unused in this round. The cache enforces a 48-hour event-time window, but prospective record availability and chunk-order Glasgow Coma Scale reconstruction have not been verified, and the length-of-stay eligibility limits applicability beyond this population.

\begin{table}[H]
\centering\small
\caption{Source role sizes in the fixed clinical patient plans (ICU stays). Validation excludes separately reserved adaptation data; target cohorts are the same across seeds.}
\label{tab:roles}
\begin{tabular}{llrr}
\toprule
Direction & Seed & Source train & Source validation\\
\midrule
E$\to$M & 42 & 18{,}555 & 6{,}039\\
E$\to$M & 3407 & 21{,}952 & 4{,}089\\
E$\to$M & 2026 & 20{,}461 & 4{,}679\\
M$\to$E & 42 & 7{,}692 & 1{,}944\\
M$\to$E & 3407 & 7{,}063 & 2{,}176\\
M$\to$E & 2026 & 8{,}509 & 1{,}691\\
\bottomrule
\end{tabular}
\end{table}

\subsection{Architecture and training}
The physiology views contain the first eleven demographic inputs, bedside measurements and laboratory features, excluding the creatinine, lactate, bilirubin and culture-site families reserved for construct checks. Their seven information blocks have dimensions $(4,2,2,2,4,8,16)$, giving 38 continuous coordinates. A separate context path receives the excluded laboratory columns, remaining static features, medications and procedures, with widths 128 and 64. The mortality head reads physiology and context through a width-128 layer; the mechanism branch adds a linear readout of two standardized state log-odds coordinates, initially weighted zero. All prediction inputs are available to every rich variant.

Five rich variants share the backbone: matched ERM; information only, the lattice without a state branch; uncorrected state, state evidence without conditioning; conditional state (\method{}); and no edges, the block-ordered mechanism with its structural edges removed. They share a five-epoch supervised warm-up with AdamW at $10^{-3}$, then independent optimizers at $3\times10^{-4}$, weight decay $10^{-4}$, batch size 256, clipping at 5 and at most 30 epochs, with patience 6 on source-validation mean group AUROC and a $10^{-4}$ improvement threshold; the warm-up checkpoint is an explicit candidate, and a fixed five-epoch averaging window around the best epoch is considered with source-only batch-norm recalculation and source validation choosing raw versus averaged weights (a local weight-averaging option, not canonical SWAD). The information objective has weight $0.01$ and the mechanism density regularizer $0.001$. The mechanism initializes with $K=3$, two EM restarts, 100 iterations and ridge $0.01$; the block order is imposed. State covariance parameters, encoder, context and outcome head update together; the selection basis is recomputed once per epoch outside gradient tracking. The feature density is a regularizer without a raw-input Jacobian, not an identified generative likelihood of clinical observations. The deployed mechanism subtracts the rank-one marginal evidence, and the source structure audit retains all 38 coordinates in all six fits. Because gradients from the supervised endpoint reach the physiology encoder, the clinical state model is label-informed through its encoder even though its initialization fit excludes $Y$; the design does not establish label-free state identification.

\subsection{Matched ablations}
Table~\ref{tab:icuabl} isolates the mechanism choices on the shared backbone. Sections~\ref{sec:exp-nfactors} and~\ref{sec:exp-icu-factors} vary the number of latent factor nodes the mechanism assumes.

\subsection{Historical baselines}
The historical rows of Table~\ref{tab:icu} are earlier source-trained implementations evaluated on the same fixed target cohorts. Early fusion uses a two-layer width-256 GELU network with dropout $0.1$; late fusion adds the logits of five such modality networks trained jointly and has a different parameter count. V-REx, GroupDRO and CORAL use local objectives based on \citet{krueger2021out,sagawa2020distributionally,sun2016deep}; ERM and the weight-averaging model sample up to 256 observations per step, while the environment-based objectives sample up to 16 hospitals with up to 32 observations each, so compute is not exactly matched. V-REx uses weight 1 for five epochs and 10 thereafter; GroupDRO a log-weight step of $0.01$; CORAL alignment weight 1. All use Adam at $10^{-3}$ with weight decay $10^{-4}$, at most 60 epochs, patience 8 and source-validation pooled AUROC, source-fitted standardization and no target adaptation. The SWAD-style baseline's backbone and schedule differ from the rich architecture, which is why the matched controls are necessary. These comparisons provide context; they cannot by themselves isolate the mechanism branch.

\subsection{Construct probes and post-fit diagnostics}
Construct probes use source-selected ridge regression with $\alpha\in\{0.1,1,10,100\}$ and age and sex controls; they read physiology or state evidence, never the context representation. The held-out laboratory families can nevertheless enter the mortality context branch. Creatinine $R^2$ is $0.307$ and $0.222$ from the continuous representation and $0.059$ and $-0.075$ from state evidence, against $0.512$ and $0.437$ for a source-fitted 38-dimensional PCA reference; these proxies are not disease states, and measured-only target subgroups may be selected by measurement practices. Post-fit diagnostics find that the first natural-shift direction explains $37.1$ to $57.6\%$ of the fitted squared singular-value energy of the clinical mechanisms; this is neither a causal dimension estimate nor proof that a larger rank would help. Training loss commonly falls while validation NLL increases, and early stopping often selects the warm-up or an early epoch; checkpoints before the warm-up were not selectable in this protocol.

\section{ADNI: protocol and complete results}
\label{app:adni}

\begin{table*}[t]
\centering\small
\caption{Complete ADNI primary results on the same 217 target patients: mean$\pm$SD over three downstream fits sharing one upstream encoder. AUROC and AUPRC are macro one-versus-rest ($\uparrow$); NLL and Brier are lower-is-better; Brier sums squared probability errors over the three classes. The lower group accesses 54 separate unlabeled target-context patients.}
\label{tab:adni}
\resizebox{\linewidth}{!}{%
\begin{tabular}{lcccc}
\toprule
Method & Macro AUROC & Macro AUPRC & NLL & Brier\\
\midrule
\multicolumn{5}{l}{\emph{Source-only prediction; no target adaptation}}\\
ERM early fusion & $0.738{\scriptstyle\pm.004}$ & $0.587{\scriptstyle\pm.002}$ & $0.917{\scriptstyle\pm.043}$ & $0.555{\scriptstyle\pm.008}$\\
ERM late fusion & $0.747{\scriptstyle\pm.008}$ & $0.596{\scriptstyle\pm.011}$ & $0.909{\scriptstyle\pm.010}$ & $0.551{\scriptstyle\pm.006}$\\
ERM MRI only & $0.582{\scriptstyle\pm.023}$ & $0.419{\scriptstyle\pm.018}$ & $1.022{\scriptstyle\pm.019}$ & $0.619{\scriptstyle\pm.012}$\\
ERM tabular only & $0.734{\scriptstyle\pm.007}$ & $0.586{\scriptstyle\pm.016}$ & $0.897{\scriptstyle\pm.020}$ & $0.543{\scriptstyle\pm.005}$\\
IRM-style & $0.728{\scriptstyle\pm.010}$ & $0.567{\scriptstyle\pm.015}$ & $0.923{\scriptstyle\pm.064}$ & $0.563{\scriptstyle\pm.026}$\\
V-REx & $0.736{\scriptstyle\pm.008}$ & $0.583{\scriptstyle\pm.011}$ & $0.940{\scriptstyle\pm.011}$ & $0.551{\scriptstyle\pm.005}$\\
GroupDRO & $0.736{\scriptstyle\pm.009}$ & $0.580{\scriptstyle\pm.012}$ & $0.906{\scriptstyle\pm.039}$ & $0.557{\scriptstyle\pm.018}$\\
CORAL & $0.736{\scriptstyle\pm.010}$ & $0.580{\scriptstyle\pm.013}$ & $0.901{\scriptstyle\pm.034}$ & $0.553{\scriptstyle\pm.015}$\\
Lattice ERM (wide) & $0.706{\scriptstyle\pm.015}$ & $0.553{\scriptstyle\pm.027}$ & $0.919{\scriptstyle\pm.016}$ & $0.557{\scriptstyle\pm.009}$\\
Lattice ERM (narrow) & $0.744{\scriptstyle\pm.004}$ & $0.576{\scriptstyle\pm.012}$ & $0.941{\scriptstyle\pm.008}$ & $0.549{\scriptstyle\pm.004}$\\
Lattice ERM (deep) & $0.698{\scriptstyle\pm.016}$ & $0.558{\scriptstyle\pm.013}$ & $0.980{\scriptstyle\pm.046}$ & $0.582{\scriptstyle\pm.021}$\\
\method{} source strict & $0.711{\scriptstyle\pm.019}$ & $0.538{\scriptstyle\pm.029}$ & $0.977{\scriptstyle\pm.025}$ & $0.579{\scriptstyle\pm.004}$\\
\method{} source common-style & $0.717{\scriptstyle\pm.002}$ & $0.564{\scriptstyle\pm.002}$ & $1.052{\scriptstyle\pm.009}$ & $0.627{\scriptstyle\pm.005}$\\
\midrule
\multicolumn{5}{l}{\emph{Independent unlabeled target context; 54 patients}}\\
Target-normalized ERM & $0.734{\scriptstyle\pm.012}$ & $0.579{\scriptstyle\pm.012}$ & $0.956{\scriptstyle\pm.047}$ & $0.567{\scriptstyle\pm.011}$\\
\method{} target strict & $0.709{\scriptstyle\pm.021}$ & $0.535{\scriptstyle\pm.035}$ & $0.973{\scriptstyle\pm.033}$ & $0.577{\scriptstyle\pm.009}$\\
\method{} target common-style & $0.715{\scriptstyle\pm.002}$ & $0.563{\scriptstyle\pm.002}$ & $1.053{\scriptstyle\pm.011}$ & $0.628{\scriptstyle\pm.006}$\\
\method{} no-environment & $0.726{\scriptstyle\pm.000}$ & $0.586{\scriptstyle\pm.000}$ & $0.891{\scriptstyle\pm.000}$ & $0.543{\scriptstyle\pm.000}$\\
\method{} affine posterior & $0.716{\scriptstyle\pm.003}$ & $0.563{\scriptstyle\pm.004}$ & $1.054{\scriptstyle\pm.014}$ & $0.629{\scriptstyle\pm.007}$\\
\method{} deep strict & $0.710{\scriptstyle\pm.005}$ & $0.588{\scriptstyle\pm.020}$ & $0.919{\scriptstyle\pm.027}$ & $0.560{\scriptstyle\pm.022}$\\
\method{} deep common-style & $0.761{\scriptstyle\pm.017}$ & $0.602{\scriptstyle\pm.011}$ & $0.937{\scriptstyle\pm.031}$ & $0.529{\scriptstyle\pm.013}$\\
\bottomrule
\end{tabular}}
\end{table*}

\paragraph*{Endpoint and cohort}
The ADNI experiment classifies contemporaneous CN, MCI and AD diagnoses; it does not predict progression. The cohort contains 1{,}807 patients: 1{,}264 source-training, 272 source-validation and 271 target patients, where within each diagnosis the oldest 15\% define an age-shift target, an age-based selection shift. This diagnosis-controlled holdout had been evaluated previously, so the analysis is exploratory. A fixed, label-blind patient-hash rule reserves 54 validation patients and 54 target patients as separate unlabeled context sets; the remaining 218 source-validation patients select downstream checkpoints, and all 20 methods share the same 217-patient primary target (56 CN, 113 MCI, 48 AD). Roles are patient-disjoint. Source environments are source-training age tertiles (420, 415 and 429 patients), not hospitals.

\paragraph*{Shared representation and replication}
All methods reuse one upstream source-supervised encoder selected on source validation (encoder seed 2026), whose four 16-dimensional views represent T1 MRI, demographics, behaviour and genetics/history, each with an availability indicator; all retained patients have MRI, and image inputs were screened for harmful content with ShieldGemma~2~\citep{zeng2025shieldgemma2} as an auxiliary content-moderation check. Standardization uses source training only. Seeds 42, 3407 and 2026 repeat downstream fitting, not upstream MRI training, so sample SD reflects downstream variability on a shared encoder and cohort. Some clinical measurements use nearest visits rather than strictly pre-MRI visits; neither prospective availability nor an independent source-validation test is claimed.

\paragraph*{Identity of the variant}
These results use the earlier joint categorical \method{} extension: a linear lattice, three Gaussian latent components, a tempered label-assisted EM objective with weight $0.25$, class-pair mean-relevance filtering across age environments, and a categorical mixture head. It differs from the common-coordinate conditional-evidence estimator and from the clinical rich-encoder experiment, does not instantiate their estimator, and its tempered objective is not a normalized joint likelihood. Strict source deployment removes complementary-coordinate shared style and source-environment head intercepts while retaining source posterior offsets; common-style variants retain the shared style term; target variants estimate posterior-branch centring from the separate 54-patient unlabeled context. Target labels are never used for fitting or checkpoint selection.

\paragraph*{Training and reporting}
Supervised networks have joint softmax outputs and are selected on source-validation multiclass NLL; MLPs use two 256-unit GELU layers, dropout $0.1$, Adam at $10^{-3}$ with weight decay $10^{-4}$, at most 60 epochs and patience 8. DG settings are fixed local implementations. State EM uses at most 200 iterations; categorical heads use 600 Adam updates at learning rate $0.02$. All 60 primary method--seed cells are complete and all 24 state-EM traces report convergence, which is not evidence of identification. An archived paired diagnosis-stratified bootstrap (2{,}000 patient draws of the mean three-seed probability vectors) conditions on fixed models, omits retraining uncertainty and is not multiplicity-adjusted; for example, target strict \method{} minus early ERM has an ensemble macro-AUROC difference of $-0.020$ with 95\% interval $[-0.063,0.022]$. Primary no-environment \method{} predicts no AD cases in any seed, so some per-class precision values are undefined; all four metrics in Table~\ref{tab:adni} have three valid seeds. The evidence supports a bounded exploratory result, not a clinical improvement claim.

\end{document}